\documentclass[10pt,twocolumn,letterpaper]{article}

\usepackage{algorithm}
\usepackage{algorithmic}

\usepackage{iccv}
\usepackage{times}
\usepackage{epsfig}
\usepackage{graphicx}
\usepackage{amsmath}
\usepackage{amssymb}

\usepackage{amsthm}

\newtheorem{thm}{Theorem}

\newtheorem{defn}[thm]{Definition}
\newtheorem{prop}[thm]{Proposition}

\usepackage[pagebackref=true,breaklinks=true,letterpaper=true,colorlinks,bookmarks=false]{hyperref}

\newif\ifTR\TRtrue

\iccvfinalcopy % *** Uncomment this line for the final submission

\long\def\ignore#1{}

\def\calA{{\cal A}}
\def\calB{{\cal B}}

\def\calD{{\cal D}}
\def\calE{{\cal E}}

\def\calG{{\cal G}}

\def\calL{{\cal L}}
\def\calN{{\cal N}}
\def\calP{{\cal P}}

\def\calV{{\cal V}}

\newcommand{\by}{\mbox{\boldmath $y$}}

\def\oo{{\bf o}}

\def\myparagraph#1{\vspace{2pt}\noindent{\bf #1~~}}

\def\iccvPaperID{1031} % *** Enter the ICCV Paper ID here

\ificcvfinal\fi
\begin{document}

%%%%%%%%% TITLE
\title{Potts model, parametric maxflow and $k$-submodular functions
}

\author{Igor Gridchyn\\
IST Austria \\
%Institution1 address\\
{\tt\small igor.gridchyn@ist.ac.at}
% For a paper whose authors are all at the same institution,
% omit the following lines up until the closing ``}''.
% Additional authors and addresses can be added with ``\and'',
% just like the second author.
% To save space, use either the email address or home page, not both
\and
Vladimir Kolmogorov\\
IST Austria\\
%First line of institution2 address\\
%{\small\url{http://www.author.org/~second}}
{\tt\small vnk@ist.ac.at}
}

\maketitle
% \thispagestyle{empty} % *** Uncomment this line for the final submission

%%%%%%%%% ABSTRACT

\ignore{
The problem of minimizing the Potts energy function frequently occurs in computer vision applications. One way to tackle this NP-hard problem was proposed by Kovtun [19, 20]. It identifies a part of an optimal solution by running k maxflow computations, where k is the number of labels. The number of "labeled" pixels can be significant in some applications, e.g. 50-93% in our tests for stereo. We show how to reduce the runtime to O(log k) maxflow computations (or one parametric maxflow computation). Furthermore, the output of our algorithm allows to speed-up the subsequent alpha expansion for the unlabeled part, or can be used as it is for time-critical applications.

To derive our technique, we generalize the algorithm of Felzenszwalb et al. [7] for Tree Metrics. We also show a connection to k-submodular functions from combinatorial optimization, and discuss k-submodular relaxations for general energy functions.
}

\begin{abstract}
The problem of minimizing the Potts energy function
 frequently occurs in computer vision applications.
One way to tackle this NP-hard problem was proposed by Kovtun~\cite{Kovtun:DAGM03,Kovtun:PhD}.
It identifies a part of an optimal solution by
running $k$ maxflow computations, where $k$ is the number of labels.
The number of ``labeled'' pixels can be significant in some applications, e.g.\ 50-93\% in our tests for stereo.
We show how to reduce the runtime to $O(\log k)$ maxflow computations
(or one {\em parametric maxflow} computation). 
%As a by-product, our algorithm also produces a labeling we demostrate that it has a good quality, and 
Furthermore, the output of our algorithm allows to speed-up the subsequent alpha expansion for the unlabeled part, or can be used as it is for time-critical applications. %We argue
%that this should be used as a standard preprocessing 
%before running the alpha expansion algorithm.

To derive our technique, we generalize the algorithm of Felzenszwalb et al.~\cite{Felzenszwalb:TreeMetrics}
for {\em Tree Metrics}. We also show a connection to {\em $k$-submodular functions}
from combinatorial optimization, and discuss {\em $k$-submodular relaxations}
for general energy functions.
\end{abstract}

%%%%%%%%% BODY TEXT
\section{Introduction}
This paper addresses the problem of minimizing an energy function
with {\em Potts} interaction terms. This energy has found
a widespread usage in computer vision after the seminal work
of Boykov et al.~\cite{BVZ:PAMI01} who proposed an efficient approximation
algorithm for this NP-hard problem called {\em alpha expansion}.

The algorithm of~\cite{BVZ:PAMI01} is based on the maxflow algorithm, also known as {\em graph cuts}.
Each iteration involves $k$ maxflow computations, where $k$ is the number of labels. 
Several techniques were proposed for improving the efficiency of these computations.
The most relevant to us is the method of Kovtun~\cite{Kovtun:DAGM03,Kovtun:PhD}
which computes a part of an optimal solution via $k$ maxflow computations.
We can then fix ``labeled'' nodes and run the alpha expansion algorithm
for the remaining nodes. Such scheme was a part of the 
``Reduce, Reuse, Recycle'' approach of Alahari et al.~\cite{Alahari:PAMI10}.

Our main contribution is to improve the efficiency of Kovtun's method
from $k$ maxflow computations to $\lceil 1+\log_2 k\rceil$ computations on graphs of equivalent sizes.
In some applications the method labels a significant fraction of nodes~\cite{Kovtun:DAGM03,Kovtun:PhD,Alahari:PAMI10}, % (e.g. TODO\%),
so our techique gives a substantial speed-up. We may get an improvement even when
there are few labeled nodes: it is reported in \cite{Alahari:PAMI10} that
using flow from Kovtun's computations always speeds up the alpha expansion algorithm for unlabeled nodes.

The idea of our approach is to cast the problem as another minimization
problem with {\em Tree Metrics}, and then generalize the algorithm of Felzenszwalb et al.~\cite{Felzenszwalb:TreeMetrics}
for Tree Metrics by allowing more general unary terms. This generalization is our second contribution.
Finally, we discuss some connections to {\em $k$-submodular functions}.

\myparagraph{Other related work}
A theoretical analysis of Kovtun's approach was given by Shekhovtsov and Hlavac~\cite{Shekhovtsov:CSC12,Shekhovtsov:PhD}.
It was shown that the method in~\cite{Kovtun:DAGM03,Kovtun:PhD} does not improve on the alpha expansion
in terms of the quality of the solution: if a node is labeled
by Kovtun's approach then the alpha expansion would produce the same solution for this node
upon convergence (assuming that all costs are unique; see~\cite{Shekhovtsov:CSC12}
for a more general statement).
Similarly, Kovtun's approach 
does not improve on the standard Schlesinger's LP relaxation of the energy \cite{Shekhovtsov:CSC12}.

We also mention the ``FastPD'' method of Komodakis et al.~\cite{Komodakis07,Komodakis08}.
The default version of FastPD for the Potts energy produces the same answer as the alpha expansion algorithm
but faster, since it maintains not only primal variables (current solution)
but also dual variables (``messages''). Intuitively, it allows to reuse flow
between different maxflow computations. An alternative method for reusing
flow was used by Alahari et al.~\cite{Alahari:PAMI10}, who reported similar speed-ups.

\section{Preliminaries}\label{sec:prelim}
The Potts energy  for labeling $x\in\calL^V$ is given by
\begin{equation}
f(x)=\sum_{i\in V}f_i(x_i)+\sum_{\{i,j\}\in E}\lambda_{ij}[x_i\ne x_j] %\qquad\forall x\in\calL^V
\label{eq:fPotts}
\end{equation}
Here $V$ is the set of nodes, $E$ is the set of edges, $\calL$ is the set of labels,
$\lambda_{ij}$ are non-negative constants and $[\cdot]$ is the {\em Iverson bracket}.
It is well-known that computing a minimizer of~\eqref{eq:fPotts} is NP-hard when $|\calL|\ge 3$ \cite{BVZ:PAMI01}.

Let us review the method of Kovtun~\cite{Kovtun:DAGM03,Kovtun:PhD} for obtaining
a part of an optimal solution.
(The method is applicable to general functions - see~\cite{Kovtun:DAGM03,Kovtun:PhD,Shekhovtsov:CSC12};
here we consider only the Potts energy, in which case the formulation simplifies considerably.)
%We consider only the Potts case; see~\cite{Kovtun:DAGM03,Kovtun:PhD,Shekhovtsov:CSC12}
%for the descriptio
%Kovtun actually considered general functions,
%but for the Potts case the formulation simplifies considerably.)
%which computes a part of an optimal solution.
For a label $a\!\in\!\calL$  denote $\bar a\!=\!\calL\!-\!\{a\}$, and  let $f_i(\bar a)\!=\!\min\limits_{b\in \bar a} f_i(b)$.
Define function $f^a:\{a,\bar a\}^V\!\rightarrow\! \mathbb R$ via
\begin{eqnarray}
f^a(y)=\sum_{i\in V}f_i(y_i)+\sum_{\{i,j\}\in E}\lambda_{ij}[y_j\ne y_i] %\qquad\forall y\in\{a,\bar a\}^V % \\
\label{eq:faKovtun}
\end{eqnarray}
%where $f_i(\bar a)=\min_{b\in \bar a} f_i(b)$. 
(A remark on notation: we typically
use letter $x$ for multi-valued labelings and $y$ for binary labelings).
\begin{thm}[\cite{Kovtun:DAGM03,Kovtun:PhD}]
Let $y\in\{a,\bar a\}^V$ be a minimizer of $f^a$. For any $x\in \calL^V$ there holds
$f(x^y)\le f(x)$ where labeling $x^y$ is defined via
$~(x^y)_i=\begin{cases}a&\mbox{if }y_i=a\\x_i&\mbox{if }y_i=\bar a\end{cases}~$ for $i\in V$.
Consequently, there exists minimizer $x^\ast\in\arg\min\limits_{x\in\calL^V} f(x)$
such that $x^\ast_i=a$ for all nodes $i\in V$ with $y_i=a$.
\label{th:Kovtun:po}
\end{thm}
Kovtun's approach requires minimizing function $f^a$ for all $a\in\calL$.
A naive way to do this is to use $k$ maxflow computations on a graph with $|V|$ nodes and $|E|$ edges, where $k=|\calL|$.
To reduce this to $O(\log k)$ maxflow computations,
we will use the following strategy. First, we will define an auxiliary
function $g:\calD^V\rightarrow\mathbb R$ where $\calD=\calL\cup\{\oo\}$, $\oo\notin\calD$.
%is a new label. 
We will then present an efficient algorithm for minimizing $g$,
and show that a minimizer $x\in\arg\min\{g(x)\:|\:x\in\calD^V\}$
determines a minimizer $y\in\arg\min\{ f^a(y)\:|\:y\in\{a,\bar a\}^V\}$
for each $a\in\calL$ in the natural way, i.e. $y_i=a$ if $x_i=a$
and $y_i=\bar a$ otherwise.
Function $g$ will have the following form:
\begin{equation}
g(x)=\sum_{i\in V} g_i(x_i) + \sum_{\{i,j\}\in E} \lambda_{ij}d(x_i,x_j)
\label{eq:PottsRelaxation}
\end{equation}
where $d(\cdot,\cdot)$ is a {\em tree metric} with respect to a certain tree $T$:
\begin{defn}
Let $T=(\calD,\calE,d)$ be a weighted undirected tree with positive weights $d(e)$, $e\in \calE$.
%The value $d(e)$ is treated as the length of $e$.
The {\em tree metric} on $\calD$ is the function $d:\calD\times\calD\rightarrow\mathbb R$
defined as follows: $d(a,b)$ for $a,b\in \calD$ is the length of the unique
path from $a$ to $b$ in $T$, where $d(e)$ is treated as the length of edge $e\in \calE$.
\end{defn}
We define $T$ as the star graph rooted at $\oo$, i.e.\ $\calE=\{\{a,\oo\}\:|\:a\in\calL\}$.
All edges are assigned length 1.
%: $d(a,\oo)=1$ for $a\in\calL$. 
The unary functions in~\eqref{eq:PottsRelaxation}
are set as follows: $g_i(\oo)=0$ and $g_i(a)=f_i(a)-f_i(\bar a)$ for $a\in\calL$.
Function $g$ in eq.~\eqref{eq:PottsRelaxation} is now completely defined.
It can be seen that minimizing $f^a$ is equivalent to minimizing $g(y)$ over $y\in\{a,\oo\}^V$.

The following observation will be crucial. %; checking this fact is straightforward.
\begin{prop}
For any $i\in V$ and $a,b\in\cal L$ with $a\ne b$ there holds $g_i(a)+g_i(b)\ge 0$.
\end{prop}
\begin{proof}
Let $a_1\in\arg\min\limits_{a\in\calL} f_i(a)$ and $a_2\in\arg\min\limits_{a\in\bar a_1} f_i(a)$.
We have $g_i(a_1)=f_i(a_1)-f_i(a_2)\le 0$ and $g_i(a)=f_i(a)-f_i(a_1)\ge f_i(a_2)-f_i(a_1)\ge 0$ for any $a\in\bar a_1$.
This implies the claim.
\end{proof}
More generally, we say that function $g_i:\calD\rightarrow \mathbb R$ is {\em $T$-convex}
if for any pair of edges $\{a,b\},\{b,c\}\in \calE$ with $a\ne c$ there holds
\begin{equation}
d(a,c)g_i(b)\le d(b,c)g_i(a)+d(a,b)g_i(c)
\end{equation}
Clearly, terms $g_i$ contructed above are $T$-convex. We will prove the following result
for an arbitrary tree $T$ and function $g$ with $T$-convex unary terms $g_i$.
(Part (a) will imply that Kovtun's approach indeed reduces to the minimization of the function $g$ above;
part (b) will motivate a divide-and-conquer algorithm for minimizing $g$.)
\begin{thm}
Let $\{a,b\}$ be an edge in $\calE$. For labeling $x\in\calD^V$ define
binary labeling $x^{[ab]}\in\{a,b\}^V$ as follows: $x^{[ab]}_i$ is the label
in $\{a,b\}$ closest to $x_i$ in $T$. \\
(a) If $x\in\calD^V$ is a minimizer of $g$ then $x^{[ab]}\in\arg\min\{g(y)\:|\:y\in\{a,b\}^V\}$. \\
(b) If $y\in\arg\min\{g(y)\:|\:y\in\{a,b\}^V\}$ then function $g$ has a minimizer 
$x\in\calD^V$ such that $x^{[ab]}=y$.
%Let $x\in\calD^V$ be a minimizer of $g$, and consider edge $\{a,b\}\in\calE$.
%Define labeling $x^{[ab]}\in\{a,b\}^V$ as follows: $x^{[ab]}_i$ is the label
%in $\{a,b\}$ closest to $x_i$ in $T$.
%Then $x^{[ab]}\in\arg\min\{g(y)\:|\:y\in\{a,b\}^V\}$.
\label{th:mult:binary:equivalence}
\end{thm}
Note, part (a) and a repeated application of Theorem~\ref{th:Kovtun:po} give that any minimizer $x\in\calD^V$ of $g$
is a partially optimal labeling for $f$, i.e.\  $f$ has a minimizer $x^\ast\in\calL^V$ such that $x^\ast_i=x_i$ for all $i$ with $x_i\ne\oo$.

In the next section we consider the case of an arbitrary tree $T$,
and present an efficient algorithm for minimizing function $g$ with $T$-convex unary terms.
%~\ref{sec:gAlg} we present a minimization algorithm for an arbitrary tree $T$.
In section~\ref{sec:details} we discuss its specialization to the star graph $T$
with unit edge length, and sketch some implementation details.
Then in section~\ref{sec:ksub} we describe a connection to {\em $k$-submodular functions}.
Section~\ref{sec:experiments} gives experimental results, and section~\ref{sec:conclusion}
presents conclusions.

%We choose tree $T$ as follows: the set of edges is $\calE=\{\{a,\oo\}\:|\:a\in\calL\}$,
%and the length of all edges in $\calE$ is set to $1/2$. It is easy to check that the function
%$g_{ij}(x_i,x_j)=\lambda_{ij}d(x_i,x_j)$ is a $k$-submodular relaxation of function $\lambda_{ij}[x_i\ne x_j]$.

\vspace{-7pt}
\section{Minimization algorithm for general $T$}\label{sec:gAlg}
\vspace{-5pt}
We build on the work of Kolen~\cite{Kolen} and Felzenszwalb et al.~\cite{Felzenszwalb:TreeMetrics}.
They considered the case when unary functions $g_i(\cdot)$ are given by
$g_i(x_i)=\lambda_i d(x_i,c_i)$ where $\lambda_i\ge 0$ and $c_i$ is a constant node in $\calD$.
\cite{Kolen} showed that such function can be minimized via $|\calD|$ maximum flow computations
on graphs with $O(|V|)$ nodes and $O(|E|)$ edges. Using a divide-and-conquer approach, 
\cite{Felzenszwalb:TreeMetrics} improved this
to $O(\log|\calD|)$ maxflow computations (plus $O(|\calD|\log|\calD|)$ time for bookkeeping).
Their algorithm can be viewed as a generalization of the algorithm in \cite{Hochbaum:ACM01,Chambolle:EMMCVPR05,Darbon:JMIV:I}
for minimizing {\em Total Variation} functionals $g(x)=\sum_i g_i(x_i)+\sum_{\{i,j\}}\lambda_{ij}|x_j-x_i|$
with convex terms $g_i$ over $x\in\{1,2,\ldots,k\}^V$
(this corresponds to the case when $T$ is a chain with unit lengths).

In this section we show that with an appropriate modification the algorithm of~\cite{Felzenszwalb:TreeMetrics} can be applied to
function~\eqref{eq:PottsRelaxation} with $T$-convex unary terms,
and present a self-contained proof of correctness.\footnote{The proof in~\cite{Felzenszwalb:TreeMetrics}
relied on results in~\cite{Kolen}, and used a different argument. In our view,
the new proof shows more clearly why the extension to $T$-convex unary terms is possible.} 

\begin{figure}[t]
\vskip 0.2in
\normalsize
\begin{center}
\begin{tabular}{c}
\includegraphics[scale=0.45]{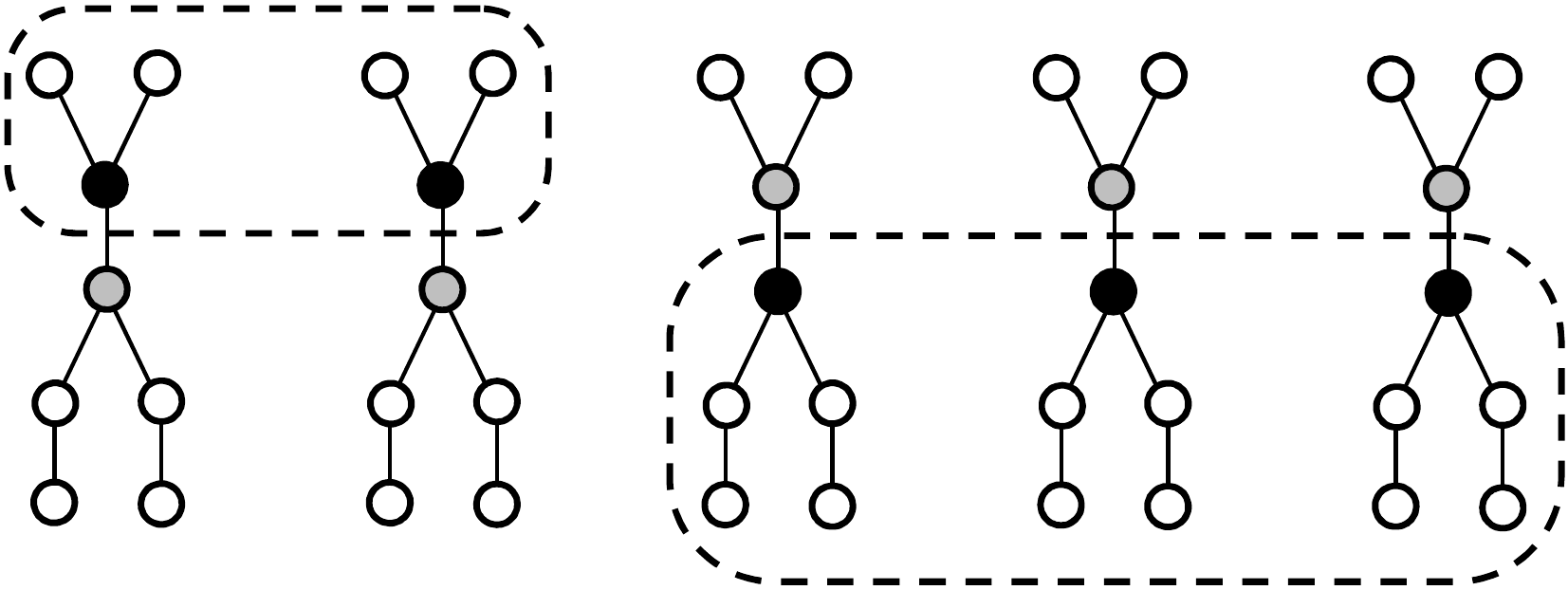}
\end{tabular}
\begin{picture}(1,0)
  \put(-5,16){$a$}
  \put(-5,1){$b$}
  \put(-190,17){$A$}
  \put(-98,1){$B$}
%  \put(-217,16){$A$}
%  \put(-128,-28){$B$}
\end{picture}
\caption{Algorithm's illustration. % for a function of 5 variables: $x=(x_1,\ldots,x_5)$. 
First, it computes $y\in\arg\min\{g(y)\:|\:y\in\{a,b\}^V\}$. Suppose
that $y=(a,a,b,b,b)$. By Theorem~\ref{th:mult:binary:equivalence}(b),
$g$ has minimizer $x$ that belongs to regions $A$ and $B$.
%This minimizer is computed via two recursive calls for $A$ and $B$.
%To compute $(x_1,x_2)$, we fix
To find solution $(x_1,x_2)$ for region $A$, the algorithm is called
recursively while fixing variables $x_3,x_4,x_5$ to $a$
(this is equivalent to fixing these variables to their
optimal labels in $B$ - the function changes only by a constant that does not depend on $x_1,x_2$).
Solution $(x_3,x_4,x_5)$ is computed similarly.
}
\label{fig:alg}
\end{center}
\vskip -0.2in
\end{figure}
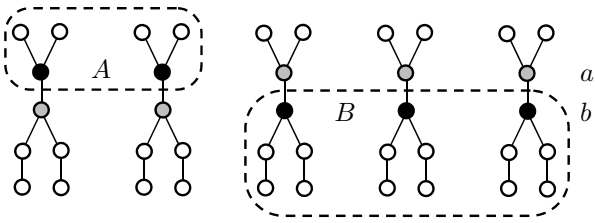

The main step of the algorithm is computing a minimizer $y\in\arg\min\{g(y)\:|\:y\in\{a,b\}^V\}$ for some edge $\{a,b\}\in \calE$
(this can be done via a maxflow algorithm). By Theorem~\ref{th:mult:binary:equivalence}(b), $y$ gives some information about a minimizer of $g$.
This information allows to split the problem into two independent subproblems
which can then be solved recursively. We arrive at the following algorithm.

\begin{algorithm}[!h]
\caption{~~{\tt SPLIT}($g$)                $$ \vspace{-32pt} $$
{\bf Input:} function $g:\calD^V\rightarrow\mathbb R$ specified by graph $(V,E)$,
tree $T=(\calD,\calE,d)$, unary terms $g_i:\calD\rightarrow\mathbb R$ and edge weights
$\lambda_{ij}$        $$ \vspace{-18pt} $$
{\bf Output:} labeling $x\in\arg\min \{g(x)\:|\:x\in\calD^V\}$
}\label{alg}
\begin{algorithmic}[1]
\STATE if $\calD=\{a\}$ return $(a,\ldots,a)$
\STATE pick edge $\{a,b\}\in \calE$ 
\STATE compute $y\in\arg\min\{g(y)\:|\:y\in\{a,b\}^V\}$ %using a maximum flow algorithm
\STATE let $T_a=(\calD_a,\calE_a,d_a)$, $T_b=(\calD_b,\calE_b,d_b)$
be the trees obtained from $T$ by removing edge $\{a,b\}$ (with $a\in\calD_a$, $b\in\calD_b$)
\FOR{$c\in\{a,b\}$}
\STATE let $V_c=\{i\in V\:|\:y_i=c\}$
\STATE let $g^c$ be the function $\calD_{c}^{V_c}\rightarrow\mathbb R$ obtained from $g$
by fixing all nodes in $V-V_{c}$ to $c$, i.e.\ 
 $g^c(x)=g(\bar x)$ where $\bar x_i=x_i$ for $i\in V_c$
and $\bar x_i=c$ for $i\in V-V_{c}$\!\!\!\!\!\!\!
\STATE let $x^c:={\tt SPLIT}(g^c)$
\ENDFOR
\STATE merge labelings $x^a$, $x^b$ into labeling $x$, return $x$
\end{algorithmic}
\end{algorithm}
Note that function $g^c$ in line 7 is defined on the subgraph
of $(V,E)$ induced by $V_c$. Indeed, for each edge $\{i,j\}\in E$ with $i\in V_c$, $j\in V-V_c$
pairwise term $\lambda_{ij}d(x_i,x_j)$ is transformed to a unary term $\lambda_{ij}d(x_i,c)$.
It can be checked that this unary term is $T_c$-convex.

The following theorem implies that the algorithm is correct; its proof is given in section~\ref{sec:proofs}.
\begin{thm}
If $x^c$ in line 9 is a minimizer of $g^c$ over $\calD_c^{V_c}$ for each $c\in\{a,b\}$
then labeling $x$ in line 10 is a minimizer of $g$ over $\calD^V$.
\label{th:PottsAlg}
\end{thm}

The algorithm leaves some freedom in line 2, namely the choice of edge $\{a,b\}\in\calE$.
Ideally, we would like to choose an edge that splits the tree into approximately equals parts
($|\calD_a|\approx|\calD_b|$). Unfortunately, this is not always possible; if, for example,
$T$ is a star graph then every split will be very unbalanced.
To deal with this issue,~\cite{Felzenszwalb:TreeMetrics} proposed to expand tree $T$ (and modify the input function accordingly)
so that the new tree $T'$ admits a more balanced split. Details are given below.

Let $a$ be a node in $\calD$ with two or more neighbors.
Let us split these neighbors
into non-empty disjoint sets $\calN$, $\calN'$
and modify tree $T$ as described in Fig.~\ref{fig:insertion}.
(This step is inserted before line 2; the new edge $\{a,b\}$ becomes the output of line 2.)
We denote $\calD'=\calD\cup\{b\}$; also, let $\calD'_{a}$, $\calD'_{b}$
be the connected components of $T'$ after removing
edge $\{a,b\}$ (with $a\in\calD'_{a}$, $b\in\calD'_{b}$).

The length of new edge $\{a,b\}$ is set to an  infinitesimally small constant $\epsilon>0$.
The new unary function $g^\epsilon_i:\calD'\rightarrow\mathbb R$ for node $i\in V$ is defined via % follows:
%\begin{equation}
%g^\epsilon_i(c)=\begin{cases}
%g_i(c) & \mbox{if }c\in\calD'_{a} \\
%g_i(a)+\epsilon\cdot u_i & \mbox{if }c=b \\
%g_i(c)+\epsilon\cdot u_i & \mbox{if }c\in\calD'_{b}-\{b\} 
%\end{cases}
%\label{eq:gepsilon}
%\end{equation}
\begin{equation}
g^\epsilon_i(c)=g_i(c)+\epsilon\cdot u_i\cdot[c\in\calD'_b] \qquad \forall c\in\calD'
\label{eq:gepsilon}
\end{equation}
where we assume that $g_i(b)=g_i(a)$,
and $u_i\in\mathbb R$ is chosen in such a way that function $g^\epsilon_i$ is $T'$-convex.
(Such $u_i$ always exists - see below). The new functional is thus $g^\epsilon(x)=\sum_{i\in V}g^\epsilon_i(x_i)
+\sum_{\{i,j\}\in E} \lambda_{ij}d^\epsilon(x_i,x_j)$
for  $x\in(\calD')^V$, where $d^\epsilon$ is the new tree metric.

There holds $|g^\epsilon(x)-g(x^{b\mapsto a})|\le const\cdot\epsilon$ for any $x\in (\calD')^V$,
where $x^{b\mapsto a}$ is the labeling obtained from $x$ by assigning label $a$ to nodes with label $b$.
Therefore, if $\epsilon$ is small enough then the following holds: if $x\in(\calD')^V$ is an optimal
solution of the modified problem then $x^{b\mapsto a}$ is an optimal solution of the original problem.

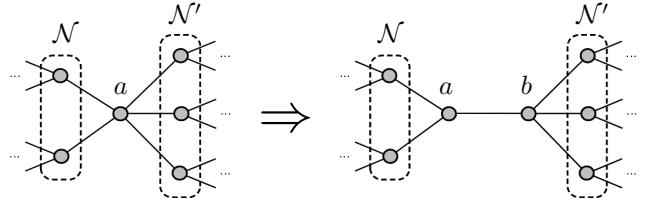
\begin{figure}[t]
\vskip 0.2in
\normalsize
\begin{center}
\begin{tabular}{c}
\includegraphics[scale=0.22]{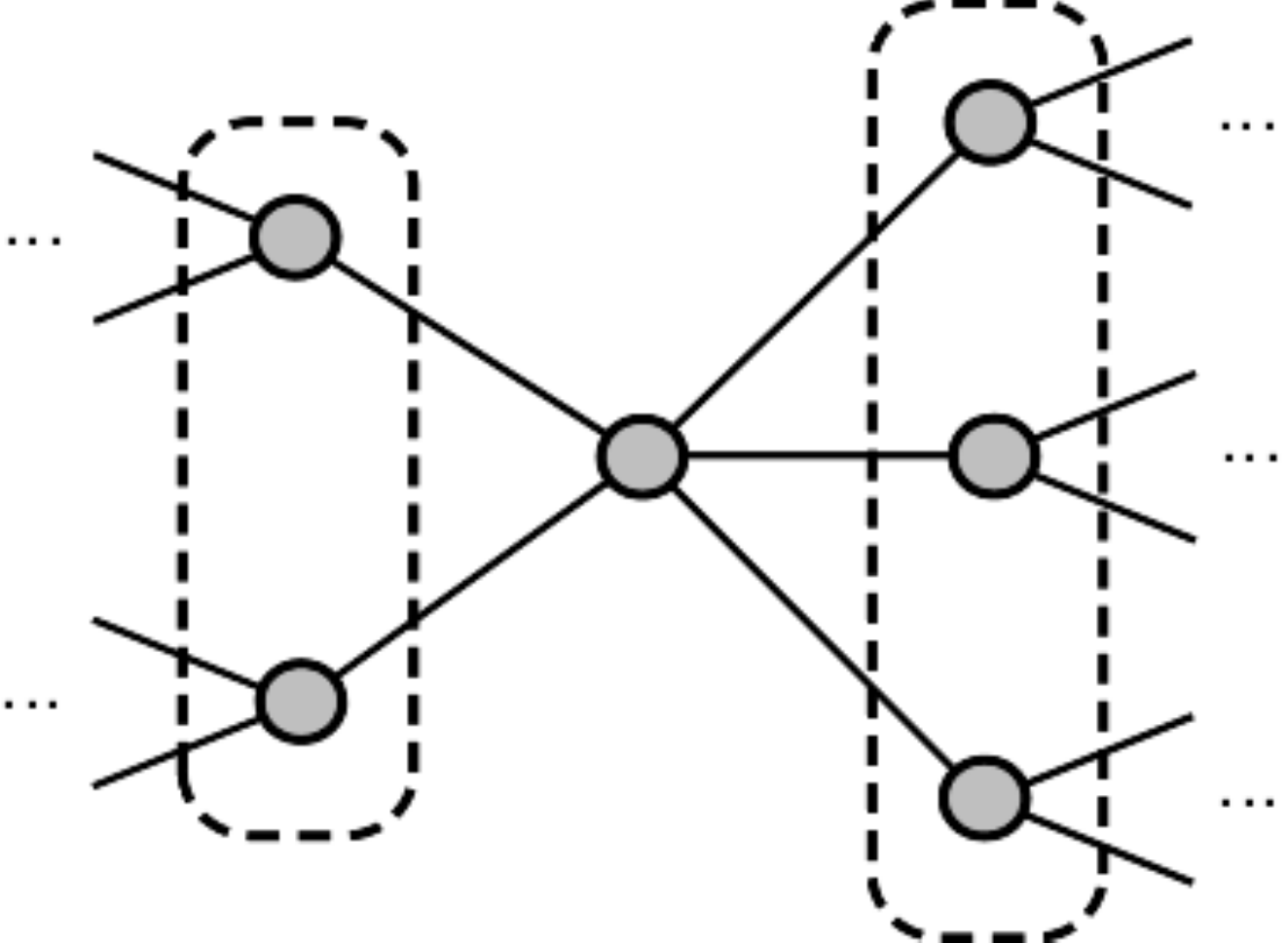}
~~~\raisebox{25pt}{\huge $\Rightarrow$}~~~
\includegraphics[scale=0.22]{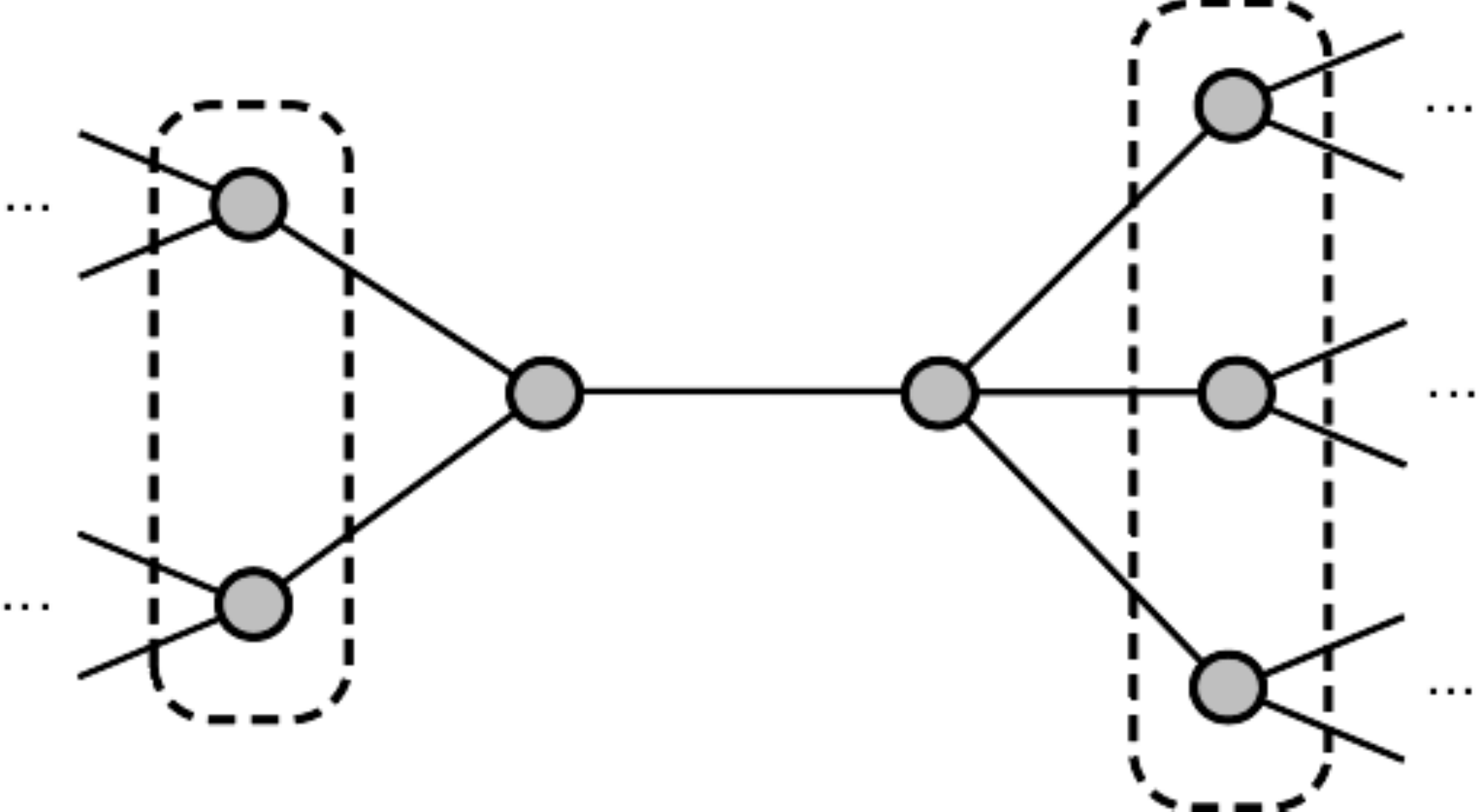}
 \vspace{-11pt} \\
\begin{picture}(1,0)
  \put(-80,40){$a$}
  \put(43,40){$a$}
  \put(74,40){$b$}
  \put(-103,60){$\calN$}
  \put(-59,67){$\calN'$}
  \put(20,60){$\calN$}
  \put(95,67){$\calN'$}
\end{picture}
\end{tabular}
\caption{Inserting edge into $T$. Given node $a\in\calD$ and the partition of its neighbors $\calN\cup\calN'$,
tree $T$ is modified as follows: (i) add new node $b\notin\calD$; (ii) add new edge $\{a,b\}$;
(iii) keep nodes $c\in\calN$ as neighbors of $a$, but make nodes $c'\in\calN'$ neighbors of $b$.
}
\label{fig:insertion}
\end{center}
\vskip -0.2in
\end{figure}

The cost function used in line 3 can be written as
$g^\epsilon(y)=const+\epsilon\cdot g'(y)$
for all $y\in\{a,b\}^V$, where function $g':\{a,b\}^V\rightarrow\mathbb R$ is defined via
\begin{equation}
%g^\epsilon(y)&\!\!=\!\!&const+\epsilon\cdot g'(y)\qquad\forall y\in\{a,b\}^V \\
g'(y)=\sum_{i\in V}u_i\cdot[y_i=b]+\sum_{\{i,j\}\in E}\lambda_{ij}\cdot[y_i\ne y_j]
\end{equation}
Therefore, minimizing $g^\epsilon$ over $\{a,b\}^V$ is equivalent to minimizing $g'$
(and thus the minimizer does not depend on $\epsilon$).

To summarize, we showed that the {\tt SPLIT} algorithm remains correct
if we replace line 2 with the tree modification step described above,
and in line 3 compute $y\in\arg\min\{g'(y)\:|\:y\in\{a,b\}^V\}$.
Also, in line 10 we need to convert labeling $x^b$ to $(x^b)^{b\mapsto a}$
before merging with $x^a$.

%Note, the operation above makes sense only when $|\calN|\ge 2$ and $|\calN'|\ge 2$
%(otherwise there is no gain - it is more efficient to split one of the incident
%edges). 
%Note, the new edge $(a,b)$ immediately gets deleted, and so we don't need
%to specify the precise value of length $\epsilon$ of this edge.

\myparagraph{Selecting $u_i$}
It remains to show that value $u_i$ for node $i\!\in\! V$ can be set in such a way that function~\eqref{eq:gepsilon}
is $T'$-convex.
\begin{prop}
Define 
$$
u_i^{\min}\!=\!-\min_{c\in\calN}\frac{g_i(c)\!-\!g_i(a)}{d(a,c)} \qquad
u_i^{\max}\!=\!\min_{c'\in\calN'}\frac{g_i(c')\!-\!g_i(a)}{d(a,c')}
$$
There holds $u_i^{\min}\le u_i^{\max}$, and for any $u_i\in[u_i^{\min},u_i^{\max}]$
and $\epsilon>0$ function $g^\epsilon_i$ in~\eqref{eq:gepsilon} is $T'$-convex.
\label{prop:uminmax}
\end{prop}
\begin{proof}
Inequality $u_i^{\min}\le u_i^{\max}$ follows from $T$-convexity of $g_i$ for pairs of edges $\{c,a\},\{a,c'\}$
($c\in\calN$, $c'\in\calN'$).

%Suppose that $u_i\ge [u_i^{\min},u_i^{\max}]$. 
Let us show $T'$-convexity of $g^\epsilon_i$ for
edges $\{c,a\},\{a,b\}$ where $c\in\calN$. 
%Clearly, adding a constant to $g_i$
%does not affect the claim, therefore we can assume w.l.o.g.\ that $g_i(a)=0$.
We need to prove that
%$$
%[d(c,a)+\epsilon]\cdot 0  \le  \epsilon \cdot g_i(c)+d(c,a)\cdot\epsilon\cdot u_i 
%$$
%
$$
[d(c,a)+\epsilon]\cdot g_i(a)  \le  \epsilon \cdot g_i(c)+d(c,a)\cdot(g_i(a)+\epsilon\cdot u_i)
$$
This is equivalent to $g_i(a)\le g_i(c) + d(a,c)\cdot u_i$. This holds since
$u_i\ge u_i^{\min}\ge (g_i(a)-g_i(c))/d(a,c)$.
The proof of $T'$-convexity of $g^\epsilon_i$ for edges $\{a,b\},\{b,c'\}$ with $c'\in\calN'$
is analogous. 
\end{proof}

%%%%%%%%%%%%%%%%%%%%%%%%%%%%%%%%%%%%%%%%%%%%%%%%%%%%%%%%%%%%%%%%%%%%%%%%%%%%%%%%%%%%%%%%%%%%%%%%%%
%%%%%%%%%%%%%%%%%%%%%%%%%%%%%%%%%%%%%%%%%%%%%%%%%%%%%%%%%%%%%%%%%%%%%%%%%%%%%%%%%%%%%%%%%%%%%%%%%%

\subsection{Proof of theorems~\ref{th:mult:binary:equivalence} and~\ref{th:PottsAlg}}\label{sec:proofs}

The proof is based on the theorem below.
Versions of this theorem in the case when $T$ is a chain with unit weights appeared in~\cite{Hochbaum:ACM01,Zalesky:JAM02,Chambolle:EMMCVPR05,Darbon:JMIV:I};
eq.~\eqref{eq:coarea} was then called the {\em coarea formula}~\cite{Chambolle:EMMCVPR05,Darbon:JMIV:I}.

In part (b) we exploit the fact that unary functions $g_i$ are $T$-convex, and make use of a well-known result about the {\em parametric maxflow} problem~\cite{Gallo:SICOMP89}.
% We also use the following operation for labels $\alpha,\beta\in\calD$
%(it is extended to labelings in $\calD^V$ component-wise):
%$$
%\alpha\downarrow \beta=\begin{cases}
%\gamma & \mbox{if }\alpha\ne\beta\mbox{ and }\{\alpha,\gamma\},\{\gamma,\beta\}\in \calE \\
%\alpha & \mbox{otherwise}
%\end{cases}
%$$
%
\begin{thm} {\bf (a)} {\em [Coarea formula]} There holds
\begin{equation}
g(x)=const+\sum_{\{a,b\}\in \calE} g(x^{[ab]})\qquad\forall x\in\calD^V
\label{eq:coarea}
\end{equation}
where $x^{[ab]}\in\{a,b\}^V$ is defined as in Theorem~\ref{th:mult:binary:equivalence}. \\
{\bf (b)} Consider edges $\{a,b\},\{b,c\}\in\calE$ with $a\ne c$.
Let $y^{bc}$ be a minimizer of $\{g(y)\:|\:y\in\{b,c\}^V\}$.
If $y^{ab}$ is a minimizer of $\{g(y)\:|\:y\in\{a,b\}^V\}$ then so is labeling
$y^{ab}\downarrow y^{bc}\in\{a,b\}^V$ where binary operation $\downarrow$ is defined component-wise via
%
%If $y^{ab}$ and $y^{bc}$ are minimizers of $\{g(y)\:|\:y\in\{a,b\}^V\}$
%and  $\{g(y)\:|\:y\in\{b,c\}^V\}$ respectively then so are labelings
%$y^{ab}\downarrow y^{bc}\in\{a,b\}^V$ and $y^{bc}\downarrow y^{ab}\in\{b,c\}^V$ 
%
%where %operation $\downarrow:\{a,b,c\}\times\{a,b,c\}\rightarrow\{a,b,c\}$ is defined via
$$
\ell\downarrow \ell'=\begin{cases}
\ell & \mbox{if }\ell'=b \\
b & \mbox{if }\ell'=c
%b & \mbox{if }(\ell,\ell')=(a,c) \\
%\ell & \mbox{otherwise}
\end{cases}
\qquad\forall \ell\in\{a,b\},\ell'\in\{b,c\}
$$
%and this operation is applied component-wise.
%
%
%%
%defined via
%\begin{eqnarray}
%~[y^{ab}\uparrow y^{bc}]_i&=&\begin{cases}
%y^{ab}_i & \mbox{if }y^{bc}_i=b \\
%b & \mbox{if }y^{bc}_i=c
%\end{cases}
%\\
%~[y^{bc}\uparrow y^{ab}]_i&=&\begin{cases}
%y^{bc}_i & \mbox{if }y^{ab}_i=b \\
%b & \mbox{if }y^{ab}_i=a
%\end{cases}
%\end{eqnarray}
\label{th:coarea:parametric}
\end{thm}
\begin{proof}
{\bf Part (a)~}
It is straightforward to check that the following holds for nodes $i\in V$ 
and edges $\{i,j\}\in E$ respectively:
\begin{eqnarray*}
g_i(x_i)&=&\left[\sum_{a\in\calD}(1-deg(a))g_i(a)\right]+\sum_{\{a,b\}\in \calE} g(x^{[ab]}_i) \\
\lambda_{ij}d(x_i,x_j)&=&\lambda_{ij}\sum_{\{a,b\}\in \calE} d(x^{[ab]}_i, x^{[ab]}_j)
\end{eqnarray*}
where $deg(a)$ is the number of neighbors of $a$ in $T$. Summing these equations gives~\eqref{eq:coarea}.

\noindent {\bf Part (b)~}
%It suffices to prove the claim for labeling $y^{ab}\downarrow y^{bc}$
%(swapping edges $\{a,b\}$ and $\{b,c\}$ would then yield the claim for the other labeling).
Let $g':\{0,1\}^V\rightarrow \mathbb R$ be the function
obtained from $g$ by associating $0\mapsto a$, $1\mapsto b$.
Similarly, let $g'':\{0,1\}^V\rightarrow \mathbb R$ be the function
obtained from $g$ by associating $0\mapsto b$, $1\mapsto c$.
We can write
\begin{eqnarray*}
h'(y)\triangleq\frac{g'(y)}{d(a,b)}= const + \sum_{i\in V}u'_iy_i+\sum_{\{i,j\}\in E}\lambda_{ij}|y_j-y_i| \\
h''(y)\triangleq\frac{g''(y)}{d(b,c)}= const + \sum_{i\in V}u''_iy_i+\sum_{\{i,j\}\in E}\lambda_{ij}|y_j-y_i|
\end{eqnarray*}
where
\begin{equation*}
u'_i=\frac{g_i(b)-g_i(a)}{d(a,b)} \qquad u''_i=\frac{g_i(c)-g_i(b)}{d(b,c)}
\end{equation*}
for $i\in V$. The $T$-convexity of $g_i$ implies that $u'_i\le u''_i$.
We need to show the following:
if $y',y''\in\{0,1\}^V$ are minimizers of $h'$ and $h''$ respectively then labeling
$y'\vee y''$ is a minimizer of $h'$. This is a well-known fact
about the parametric maxflow problem~(\cite{Gallo:SICOMP89}, Lemma 2.8).
Indeed,
%\end{proof}
\begin{eqnarray*}
&&\hspace{-10pt} h'(y'\vee y'')-h'(y') \le h'(y'')-h'(y'\wedge y'')  \\ 
&&\hspace{-5pt} =
h''(y'')-h''(y'\wedge y'')+\sum_{i:(y'_i,y''_i)=(0,1)}[u'_i-u''_i]\le 0 
\end{eqnarray*}

%The lemma can now be equivalently reformulated as follows:
%if $y^{ab},y^{bc}\in\{0,1\}^V$ are minimizers of $g^{ab}$
%and $g^{bc}$ respectively 
%$g^{ab}(y)=g(\tilde y}
\end{proof}

We say that a family of binary labelings $\by=(y^{ab}\in\{a,b\}^V\:|\:\{a,b\}\in\calE\})$
is {\em consistent} if there exists labeling $x\in\calD^V$ such that $x^{[ab]}=y^{ab}$ for all $\{a,b\}\in\calE$.
%(We usually use letter $x$ for multi-valued labelings and letter $y$ for binary labelings.)
Theorem~\ref{th:coarea:parametric}(a) implies that the minimization of $g(x)$ over $x\in\calD^V$ is
equivalent to the minimization of 
\begin{equation}
G(\by)=\sum_{\{a,b\}\in\calE}g(y^{ab})
\label{eq:G}
\end{equation}
over consistent labelings $\by=(y^{ab}\in\{a,b\}^V\:|\:\{a,b\}\in\calE)$. %$y^{ab}\in\{a,b\}^V$ for edges $(a,b)\in\calE$.
Next, we analyze the consistency constraint.
\begin{prop}
Family $\by$ is consistent iff for every 
for any pair of edges $\{a,b\},\{b,c\}\in\calE$
with $a\ne c$ and any node $i\in V$ there holds $(y^{ab}_i,y^{bc}_i)\ne (a,c)$.
\label{prop:consistency}
\end{prop}
\begin{proof}
Let us fix a node $i\in V$, 
and denote $\by_i=(y^{ab}_i\:|\:\{a,b\}\in\calE\})$. 
Clearly, there is one-to-one
corrrespondence between possible labelings $\by_i$ and orientations of tree $T$.
Namely, to each $\by_i$ we associate a directed graph $\calG[\by_i]\!=\!(\calD,\vec\calE[\by_i])$
with $\vec\calE[\by_i]\!=\!\{(a,b)\:|\:\{a,b\}\!\in\!\calE, y^{ab}_i\!=\!b\}$.\!\!\!\!\!

It  can be seen that  $\by_i$ is consistent
(i.e. there exists $x_i\in\calD$
with $y^{ab}_i=x^{[ab]}_i$ for all $\{a,b\}\in\calE$)
iff graph $G[\by_i]$ has exactly one sink, i.e.\ a node without outgoing edges.
This is equivalent to the condition that each node $a\in\calD$
has at most one outgoing edge in $G[\by_i]$.
This is exactly what the condition in the proposition encodes.
%(If there are two sinks then some node the path between 
\end{proof}

We can now prove Theorems~\ref{th:mult:binary:equivalence} and~\ref{th:PottsAlg}.

\myparagraph{Theorem~\ref{th:mult:binary:equivalence}(b)}
%\begin{lemma}
%Consider edge $\{a,b\}\in\calE$ and labeling $y^{ab}\in\arg\min \{g(y)\:|\:y\in\{a,b\}^V\}$.
%There exists $x\in\arg\min \{g(x)\:|\:x\in\calD^V\}$ such that $x^{[ab]}=y^{ab}$.
%\label{lemma:proof}
%\end{lemma}
%\begin{proof}
Consider the following algorithm for constructing a family of binary labelings $\by$.
Initially, we set $\by=(y^{ab})$ where $y^{ab}=y$ is the labeling chosen in Theorem~\ref{th:mult:binary:equivalence}(b).
We also initialize subtree $T'=(\calD',\calE')$ of $T$ via $\calD'=\{a,b\}$, $\calE'=\{\{a,b\}\}$,
and then repeat the following while $T'\ne T$:
(i) pick edge $\{a',b'\}\in\calE-\calE'$ with $a'\in\calD-\calD'$, $b'\in\calD'$,
add $a'$ to $\calD'$ and $\{a',b'\}$ to $\calE'$;
(ii) pick $y^{a'b'}\in\arg\min\{g(y)\:|\:y\in\{a',b'\}^V\}$;
(iii) go through edges $\{b',c'\}\in\calE'$ with $c'\ne a'$ (in some order)
and replace $y^{a'b'}$ with $y^{a'b'}\downarrow y^{b'c'}$.

By Theorem~\ref{th:coarea:parametric}(b), the constructed family of binary labelings $\by$ satisfies
the following: $y^{a'b'}\in \arg\min\{g(y)\:|\:y\in\{a',b'\}^V\}$ for all $\{a',b'\}\in\calE$.
Using Proposition~\ref{prop:consistency}, it is also easy to check that family $\by$ is consistent;
let $x\in\calD^V$ be the corresponding labeling.
Theorem~\ref{th:coarea:parametric}(a) implies that $x$ is a minimizer of $g$.

\myparagraph{Theorem~\ref{th:mult:binary:equivalence}(a)} If all labelings $x\in\calD^V$ have unique costs $g(x)$ then
the claim follows from Theorem~\ref{th:mult:binary:equivalence}(b).
The general case can be reduced to the case above by adding 
function $\delta g(x')=\sum_{i\in V}\epsilon_i d(x_i,x'_i)$ to $g$
where $\epsilon_i>0$ are infinitesimally small numbers
and $x\in\arg\min\limits_{x\in\calD^V} g(x)$ is the labeling chosen in Theorem \ref{th:mult:binary:equivalence}.

\myparagraph{Theorem~\ref{th:PottsAlg}}
\ifTR
Let $y\in\{a,b\}^V$ be the output of line 3 and $x^a\in\calD_a^{V_a}$, $x^b\in\calD_b^{V_b}$ be minimizers of $g^a$ and $g^b$ respectively.
We will write labelings in $\calD^V$ as $(z^a,z^b)$ where $z^a\in\calD^{V_a}$ and $z^b\in\calD^{V_b}$.

Let $\hat x=(\hat x^a,\hat x^b)\in\calD^V$ be a minimizer of $g$ with $\hat x^{[ab]}=y$
(it exists by Theorem~\ref{th:mult:binary:equivalence}(b)).
Condition $\hat x^{[ab]}=y$ implies that
 $\hat x^a\in\calD_a^{V_a}$ and $\hat x^b\in\calD_a^{V_b}$.
Let us prove that $(x^a,\hat x^b)$ is a minimizer of $g$.
%First, observe tha
Let  $\hat g^a:\calD_a^{V_a}$ be the function $\hat g^a(z^a)=g(z^a,\hat x^b)$.
It can be checked that $\hat g^a(z^a)=g^a(z^a)+const$ for all $z^a\in\calD_a^{V_a}$.
Therefore, $x^a$ is a minimizer of $\hat g^a(z^a)$, and so 
$$
g(x^a,\hat x^b)=\hat g(x^a)\le \hat g(\hat x^a)=g(\hat x^a,\hat x^b)=\min_{x\in\calD^V}g(x)
$$
%$g^a(x^a)\le g^a(\hat x^a)$
A completely analogous argument shows that $(x^a,x^b)$ is a minimizer of $g$ as well.
\else
This theorem follows from Theorem \ref{th:mult:binary:equivalence}(b); details are given in the suppl. material.
\fi

\section{Implementation details}\label{sec:details}
In this section we sketch implementation details
of Algorithm \ref{alg} applied to the function constructed in section~\ref{sec:prelim}
(so $T$ is a star graph with nodes $\calD=\calL\cup\{\oo\}$).
We will discuss, in particular, how to extract optimal flows.
%The main goal is to describe how we extract maximum flow
%for computations

We use the edge insertion operation at each call of {\tt SPLIT}
except when $\calD=\{a,\oo\}$ for some $a\in\calL$.
Thus, computations can be described in terms of a binary tree 
whose nodes correspond to subsets of labels $\calA\subseteq\calL$
(Fig.~\ref{fig:binaryTree}). Let $\Omega$ be the set of nodes of this tree.
For each $\calA\in\Omega$ we run a maxflow algorithm; let $V_\calA\subseteq V$ be the set of nodes
involved in this computation.
Note that sets $V_\calA$ for nodes $\calA$ at a fixed depth
form a disjoint union of $V$ (except possibly the last level).
Therefore, these maxflow computations can be treated
as a single maxflow on the graph of the original size.
The total number of such computations is $\lceil 1+\log_2 k\rceil$ (the number of levels of the tree).
%Let $\Omega_d$ be the set of nodes at depth $d$.
%Clearly $V$ is a disjoint union of sets $V_\calA$ over $\calA\in\calO_d$

\begin{figure}[t]
\vskip 0.0in
\normalsize
\begin{center}
\begin{tabular}{c}
\includegraphics[scale=0.19]{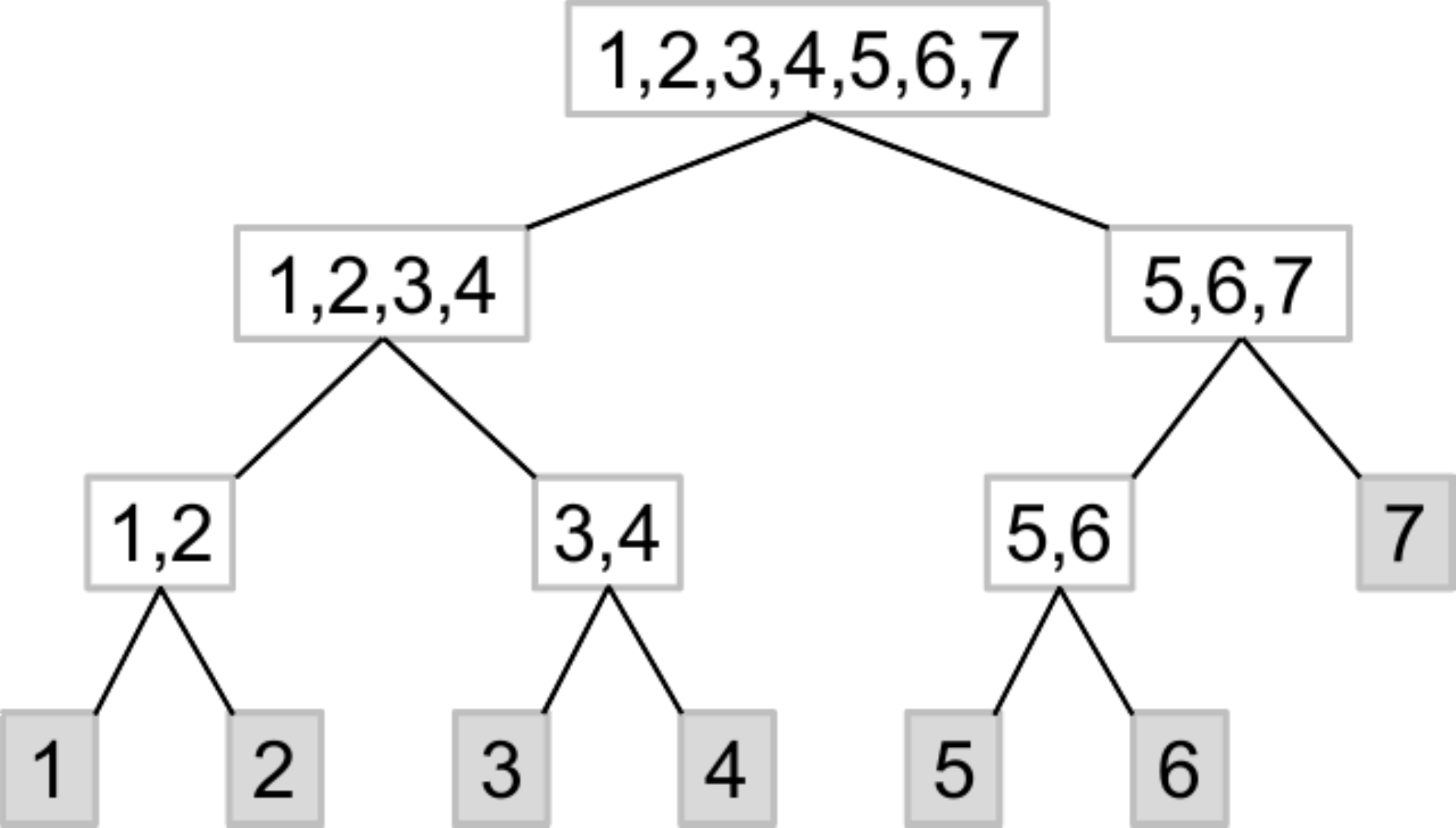}
\end{tabular}
\caption{Binary tree for the set $\calL=\{1,\ldots,7\}$.
Each node is a subset $\calA\subseteq\calL$; $\calL$
is the root and singleton subsets are the leaves.
}
\label{fig:binaryTree}
\end{center}
\vskip -0.2in
\end{figure}

For each $\calA\in\Omega$ we set up a graph with the set of nodes $V_\calA\cup\{s,t\}$ and the cut function
\begin{equation*}
f_\calA(S\cup\{s\},T\cup\{t\})\!=\!\sum_{i\in\calV_\calA}u^\calA_i[i\!\in\! T] + \sum_{\{i,j\}}\lambda_{ij}[i\!\in\! S,j\!\in\! T]
\end{equation*}
($S\stackrel{.}{\cup} T=V_\calA$).
To define $u^\calA_i$, we need to specify the meaning of
the source $s$ and the sink $t$. For non-leaf nodes $\calA\in\Omega$
the source corresponds to the left child $\calA_\ell$ 
and the sink corresponds to the right child $\calA_r$; we then have
\begin{equation}
u^\calA_i\in[g_i(\oo)-\min_{a\in\calA_{\ell}}g_i(a),-g_i(\oo)+\min_{a\in\calA_{r}}g_i(a)]
\label{eq:NGFNAGJNSAF}
\end{equation}
where we use the current value of $g_i(\oo)$ (it is zero initially and then gets decreased).
For a leaf $\calA=\{a\}$ we use a different intepretation: $s$ corresponds to label $a$
and $t$ corresponds to label $\oo$, therefore $u^\calA_i=g_i(\oo)-g_i(a)$.

We perform all maxflow computations on a single graph.
We use the Boykov-Kolmogorov algorithm~\cite{BK:PAMI04} with flow and search trees recycling~\cite{Kohli:ICCV05}.
We maintain values $u_i$ for nodes $i\in V$ that give
the current cut functions encoded by the residual graph.
After computing maxflow at a non-leaf node $\calA$ the residual graph is modified as follows.
First, for each arc $(i\rightarrow j)$ from 
the source to the sink component we do the following:
\begin{list}{$\bullet$}{\leftmargin=1.4em \itemindent=0em \itemsep=-2pt}
\item[1.] Set $u_i:=u_i-\lambda_{ij}$ and $u_j:=u_j+\lambda_{ij}$;
this simulates pushing flow $\lambda_{ij}$ along the path $t\rightarrow j\rightarrow i\rightarrow s$.
\item[2.] Remove arcs $(i\rightarrow j),(j\rightarrow i)$ from the graph.
\item[3.] Update $g_i(\oo):=g_i(\oo)-\lambda_{ij}$. 
\end{list}
Now we need to set unary costs for maxflow computations at the children $\calA_\ell$, $\calA_r$ of $\calA$.
Consider node $i\in V_{\calA_c}$, $c\in\{\ell,r\}$.
First, we compute the appropriate value $u^{\calA_c}_i$;
if $\calA_c$ is not a leaf then we compute interval~\eqref{eq:NGFNAGJNSAF}
for $\calA_c$ and choose the value $u^{\calA_c}_i$ from the interval closest to $u_i$.\footnote{It can
be shown that we only need to know $a^\ast\in\arg\min_{a\in\calL}g_i(a)$, $g_i(a^\ast)$ and $g_i(\oo)$ for that.}
%(It can be shown that we only need to consider the second and the second smallest values of $f_i$ on $\calL$ for that.)
Then we change the graph by adding $\delta^{\calA_c}_i=u^{\calA_c}_i-u_i$ to the capacity of $(s\rightarrow i)$
(or subtracting from the capacity of $(i\rightarrow t)$),
and update $u_i:=u^{\calA_c}_i$.

\myparagraph{Remark 1} \small 
The following property can be shown.
Suppose that node $i\in V$ ended up at a leaf $\{a\}\in\Omega$.
Let $\calP$ be the path from $\calL$ to $\{a\}$,
and define values $c^\calA_i$ for $\calA\in\calP$
so that $c^\calL_i=u^\calL_i$ 
and $c^\calB_i=c^\calA_i+\delta^\calB_i$ for edges $(\calA,\calB)\in\calP$.
Then values $c^\calA_i$ for nodes $\calA\in\calP-\{a\}$ are non-decreasing
w.r.t. the {\bf inorder} of the binary tree.\footnote{The monotonicity would also hold for the leaf $\{a\}$
if we changed the meaning of the source and the sink for computations at the leaves $\{a\}\in\Omega$
that are right children. However, we found it more convenient to use our interpretation.}

Such monotonicity
implies that computations at non-leaf nodes fall into the framework
of {\em parametric maxflow} of Gallo et al.~\cite{Gallo:SICOMP89}. As shown in \cite{Gallo:SICOMP89},
all computations can be done with the same worst-case complexity as
a single maxflow computation. However, this requires a more complex implementation, namely
running in parallel two push-relabel algorithms. Experiments in~\cite{Babenko:2007} suggest
that this is less efficient than a naive scheme.
\normalsize

\myparagraph{Extracting flows} Let us fix label $a\in\calL$.
Recall that Algorithm \ref{alg} yields the minimum of function $f^a$ given by~\eqref{eq:faKovtun}.
An important question is how to obtain an optimal flow
correspoding to this computation; as reported in~\cite{Alahari:PAMI10}, using this flow speeds up
the alpha expansion algorithm.

It suffices to specify the flow $\xi_{ij}$ for each arc $(i\rightarrow j)$ with $\{i,j\}\in E$ (the flow from the source and to the sink
can then be easily computed). We used the following rule.
For each edge we store flow $\xi''_{ij}$ after the final maxflow and flow $\xi'_{ij}$ immediately before
maxflows at the leaves. For each node $i\in V$ we also store leaf $\calA_i\in\Omega$
at which node $i$ ended up.
We now set flow $\xi_{ij}$ as follows:
\begin{list}{$\bullet$}{\leftmargin=1em \itemindent=0em \itemsep=-2pt}
\item if $\calA_i=\calA_j=\{a\}$ set $\xi_{ij}:=\xi''_{ij}$
\item otherwise let $\calA\in\Omega$ be the least common ancestor of $\calA_i, \calA_j, \{a\}$ in
the binary tree. If $\{a\}$ belongs to the left subtree of $\calA$ then set $\xi_{ij}:=\xi'_{ij}$,
otherwise set $\xi_{ij}:=\xi'_{ji}$.
\end{list}
We hope to prove to correctness of this procedure in a future publication;
at the moment we state it as a conjecture that was verified experimentally.
We mention that we were not able to find a scheme that would store only one flow per arc.

\section{Relation to $k$-submodular functions}\label{sec:ksub}
In this section we discuss some connections between techniques described earlier
and {\em $k$-submodular functions} introduced in~\cite{Kolmogorov:MFCS11,Huber:ISCO12}.
We also define {\em $k$-submodular relaxations} of discrete functions $f:\calL^V\rightarrow\mathbb R$
which generalize {\em bisubmodular relaxations}~\cite{Kolmogorov:bisubmodularRelaxation} of pseudo-Boolean functions
in a natural way. 
%As above, we denote $\calD=\calL\cup\{\oo\}$ where $\oo\notin\calL$.

%Let us recall the definition of a $k$-submodu

%\myparagraph{$k$-submodular functions~}
%We will treat set $\calD=\calL\cup\{\oo\}$ as a rooted tree: $\oo$ is the root with $k=|\calL|$ children.
%Let $\preceq$ be the partial order on this tree, with $\oo\preceq a$ for all $a\in \calD$.
%We define two binary operations $\sqcap,\sqcup:\calD\times\calD\rightarrow \calD$
%as follows:
%\begin{list}{$\bullet$}{\leftmargin=1em \itemindent=0em \itemsep=-2pt}
%\item if $a,b\in\calD$ are comparable then $a\sqcap b=\min\{a,b\}$, $a\sqcup b=\max\{a,b\}$
%where $\min$ and $\max$ are taken w.r.t. $\preceq$;
%\item otherwise (i.e.\ if $a,b\in \calL$, $a\ne b$) $a\sqcap b=a\sqcup b=\oo$.
%\end{list}
\begin{defn}[$k$-submodularity]
Let $\preceq$ be the partial order on $\calD=\calL\cup\{\oo\}$ such that $a\prec b$ iff $a=\oo$ and $b\in\calL$.
Define binary operations $\sqcap,\sqcup:\calD\times\calD\rightarrow \calD$ via
\begin{equation*}
(a\sqcap b,a\sqcup b)=\begin{cases}
(\oo,\oo) & \mbox{if } a,b\in\calL, a\ne b \\
(\min \{a,b\},\max \{a,b\}) & \mbox{otherwise}
\end{cases}
\end{equation*}
where $\min$ and $\max$ are taken w.r.t. partial order $\preceq$.
Function $g:\calD^V\rightarrow\mathbb R$ is $\mbox{called {\em $k$-submodular} (with $k=|\calL|$) if}$
\begin{equation}
g(x\sqcap y)+g(x\sqcup y)\le g(x)+g(y)\qquad \forall x,y\in \calD^V
\end{equation}
where operations $\sqcap,\sqcup$ are applied component-wise.
\end{defn}
It is easy to check that function $g$ constructed in section~\ref{sec:prelim}
is $k$-submodular. Another way to obtain a $k$-submodular function is as follows.
Consider some function $f:\calL^V\rightarrow \mathbb R$.
We say that function $g:\calD^V\rightarrow\mathbb R$ is a {\em $k$-submodular relaxation} of $f$
if  $g(x)=f(x)$ for all $x\in \calL^V$, and function $g$ is $k$-submodular.
It can be seen that any function $f:\calL^V\rightarrow\mathbb R$ admits a $k$-submodular
relaxation; we can set, for example, $g(x)=f(x)$ for $x\in\calL^V$ and $g(x)=C$ for $x\in\calD^V-\calL^V$,
where $C\le \min_{x\in \calL^V}f(x)$.

$k$-submodular relaxations for $k=2$ have been studied in~\cite{Kolmogorov:bisubmodularRelaxation}
under the name {\em bisubmodular relaxations}. It was shown that if $f$ is a {\em quadratic} pseudo-Boolean
function then the tightest bisubmodular relaxation is equivalent to the roof duality relaxation~\cite{Hammer:MP84}.
It was also proved that bisubmodular relaxations possess the {\em persistency}, or {\em partial optimality} property.
The argument of~\cite{Kolmogorov:bisubmodularRelaxation} extends trivially to $k$-submodular relaxations, 
as the following proposition shows.
\begin{prop}
Let $g$ be a $k$-submodular relaxation of $f$ and 
 $y^\ast\in\calD^V$ be a minimizer of $g$.
Function $f$ has a minimizer $x^\ast\in\calL^V$ such
that $x^\ast_i=y^\ast_i$ for all $i\in V$ with $y^\ast_i\in\calL$.
\end{prop}
\begin{proof}
First, observe that for any $z\in\calD^V$ there holds
$g(z\sqcup y^\ast)\le g(z)$ since $g(z\sqcup y^\ast)- g(z)\le g(y^\ast)-g(z\sqcap y^\ast)\le 0$.

Let $x\in\calL^V$ be a minimizer of $f$, and define $x^\ast=(x\sqcup y^\ast)\sqcup y^\ast$.
It can be checked that $x^\ast_i=y^\ast_i$ if $y^\ast_i\in\calL$, and $x^\ast_i=x_i$ if $y^\ast_i=\oo$.
Thus, $x^\ast\in\calL^V$. Labeling $x^\ast$ is a minimizer of $f$
since $$f(x^\ast)=g((x\sqcup y^\ast)\sqcup y^\ast)\le g(x\sqcup y^\ast)\le g(x)=f(x)$$
\end{proof}

Thus, $k$-submodular relaxations can be viewed as a generalization of the roof duality relaxation
to the case of multiple labels. Recently, Thapper and \v{Z}ivn\'y showed~\cite{ThapperZivny:FOCS12} 
that a $k$-submodular function $g$ can be minimized in polynomial time
if $g$ is represented as a sum of low-order $k$-submodular terms. (This was proved by showing
the tightness of the {\em Basic LP relaxation} (BLP); when $g$ is a sum of unary and pairwise terms,
BLP is equivalent to the standard Schlesinger's LP~\cite{Werner:PAMI07}.)
This suggests a new possibility for obtaining partial optimality for discrete functions $f$.

\myparagraph{Potts model} Let us compare the approach above
with the Kovtun's approach in the case of the Potts energy function $f$ from eq.~\eqref{eq:fPotts}.
A natural $k$-submodular relaxation of $f$ is the function
\begin{equation}
\tilde g(x)=\sum_{i\in V}\tilde g_i(x_i)+\frac{1}{2}\sum_{\{i,j\}\in E}\lambda_{ij} d(x_i,x_j)
\label{eq:GASIFNALKSGFA}
\end{equation}
where $\tilde g_i$ is a $k$-submodular relaxation of $f_i$ and $d$ is the tree metric used
in section~\ref{sec:prelim}.
It is natural to set $\tilde g_i(\oo)$ to the maximum possible value such that $\tilde g_i$ is $k$-submodular;
this is achieved by $\tilde g_i(\oo)=\frac{1}{2}[f_i(a_1)+f_i(a_2)]$
where $f_i(a_1)$ is the smallest value of $f_i$ and $f(a_2)$ is the second smallest.

The proposition below shows that minimizing $\tilde g$ yields 
the same or fewer number of labeled nodes compared to the Kovtun's approach.
\ifTR
\else
Its proof is given in the suppl. material.
\fi
\begin{prop}
Let $g$ be the function~\eqref{eq:PottsRelaxation}
corresponding to the Kovtun's approach, and $\tilde g$ be the $k$-submodular relaxation of $f$ given by~\eqref{eq:GASIFNALKSGFA}.
Assume for simplicity that $g$ and $\tilde g$ have unique minimizers $x$ and $\tilde x$ respectively.
If $\tilde x_i=a\ne \oo$ for node $i\in V$ then $x_i=a$.
\label{prop:proof}
\end{prop}
\ifTR
\begin{proof}
Let us define functions $f^a,\tilde f^a:\{a,\oo\}^V\rightarrow \mathbb R$ via
\begin{eqnarray*}
f^a(y)&=&\sum_{i\in V} u_i [y_i=a] + \sum_{\{i,j\}\in E} \lambda_{ij}[y_i\ne y_i] \\
\tilde f^a(y)&=&\sum_{i\in V} \tilde u_i [y_i=a] + \sum_{\{i,j\}\in E} \lambda_{ij}[y_i\ne y_i]
\end{eqnarray*}
where $u_i=g_i(a)-g_i(\oo)$ and $\tilde u_i=2[\tilde g_i(a)-\tilde g_i(\oo)]$.
By Theorem \ref{th:mult:binary:equivalence}, labelings $y= x^{[a\oo]}$ and $\tilde y=\tilde x^{[a\oo]}$ are unique minimizers of $f^a$ and $\tilde f^a$ respectively. 
There holds $u_i\le\tilde u_i$ for all $i\in V$ since
\begin{eqnarray*}
u_i-\tilde u_i&=&[f_i(a)-f_i(\bar a)] - 2\left[f_i(a)-\frac{f_i(a_1)+f_i(a_2)}{2}\right] \\
 &=&[f_i(a_1)+f_i(a_2)] - [f_i(a)+f_i(\bar a)]\;\le\; 0
\end{eqnarray*}
Therefore, the claim above follows from the standard result about the parametric maxflow (\cite{Gallo:SICOMP89}, Lemma 2.8;
see also the proof of Theorem~\ref{th:coarea:parametric}(b)).
\end{proof}
\else
\vspace{-5pt}
\fi
Although a $k$-submodular relaxation of the Potts energy turns out
to be worse than Kovtun's approach, 
%the proof of Proposition~\ref{prop:proof} shows that
 there are clear similarities between the two
(e.g.\ they can be solved by the same technique).
We believe that exploring both approaches (or their combination)
can be a fruitful direction for obtaining partial optimality for more general functions.

\vspace{-7pt}
\section{Experimental results}\label{sec:experiments}
\vspace{-4pt}

We applied our technique to the stereo segmentation problem on the Middlebury data~\cite{Scharstein2002,Scharstein2003,Scharstein2007}.
Computations consist of two phases: (1) solve the Kovtun's approach, and (2) run the alpha-expansion algorithm
for the unlabeled (or ``{\em non-persistent}'') part until convergence. For phase 1 we compared the speed of our algorithm (which we call ``{\em $k$-sub Kovtun}'')
with the `Reduce' method of Alahari et al.~\cite{Alahari:PAMI10}.
For phase 2 we used the FastPD method of Komodakis et al.~\cite{Komodakis07, Komodakis08}. % which is faster that the original alpha-expansion method \cite{BVZ:PAMI01}.
We used original implementations from~\cite{Alahari:PAMI10} and~\cite{Komodakis07, Komodakis08} and
%All experiments were run on 
a Core i7 machine with 2.3GHz.

As a by-product, $k$-sub Kovtun produces a labeling which we call a {\em Kovtun labeling}: pixel $i$ is assigned the label $a$ where it ended up,
as described in Sec.~\ref{sec:details}. Empirically, this labeling has a good quality - see below.

%%%%%%%%%%%%%%%%%%%%%%%%%%%%%%%%%%%%%%%%%%%%%%%%%%%%%%%%%%%%%%%%%%%%%%%%
\ifTR

\begin{figure*}
\begin{tabular}{@{\hspace{-2pt}}c@{\hspace{4pt}}c@{\hspace{4pt}}c}
   \includegraphics[width=0.33\linewidth]{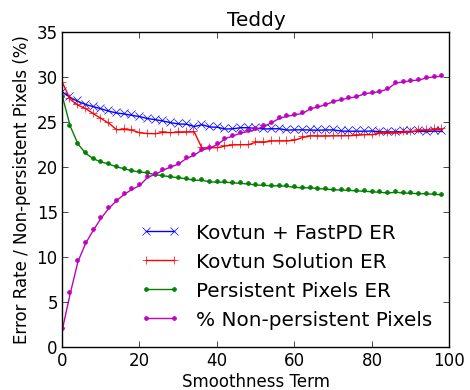} &
   \includegraphics[width=0.33\linewidth]{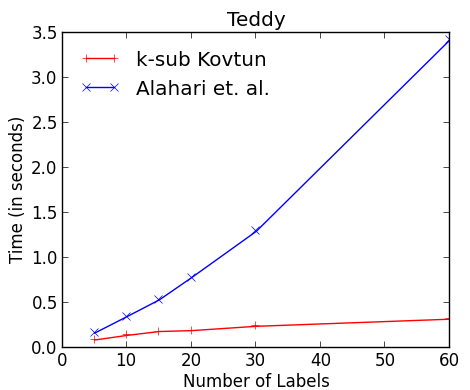} &
   \includegraphics[width=0.33\linewidth]{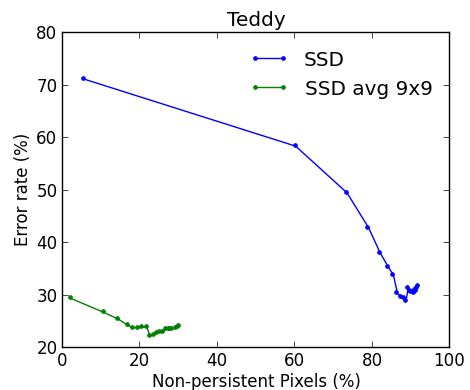} \\
\\
   \includegraphics[width=0.33\linewidth]{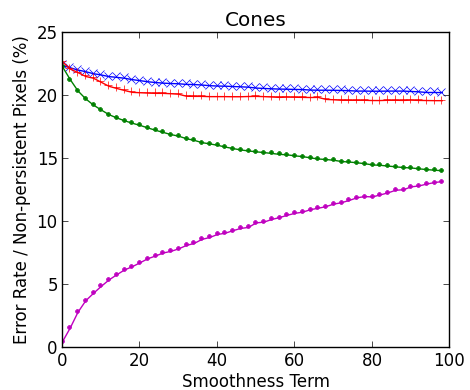} &
   \includegraphics[width=0.33\linewidth]{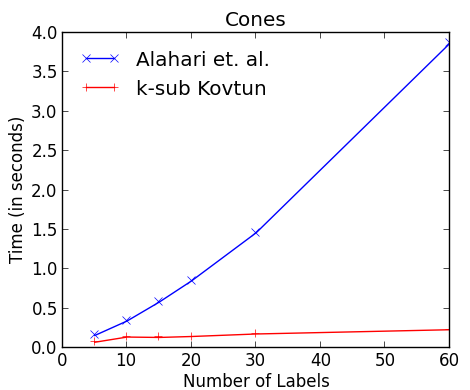} &
   \includegraphics[width=0.33\linewidth]{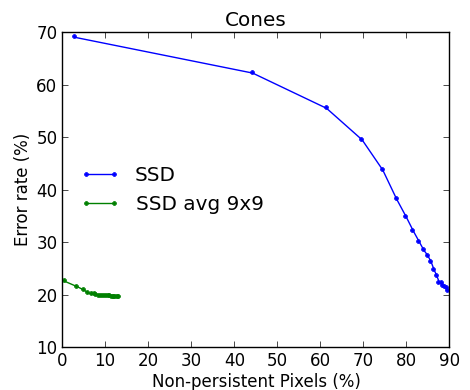} \\
\\
   \includegraphics[width=0.33\linewidth]{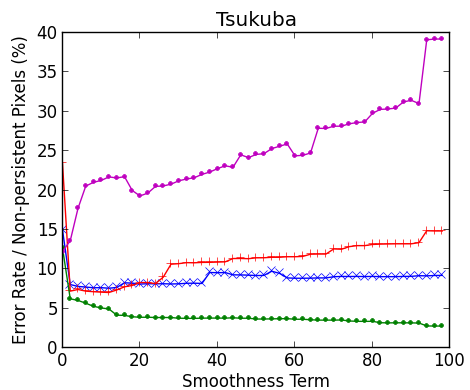} &
   \includegraphics[width=0.33\linewidth]{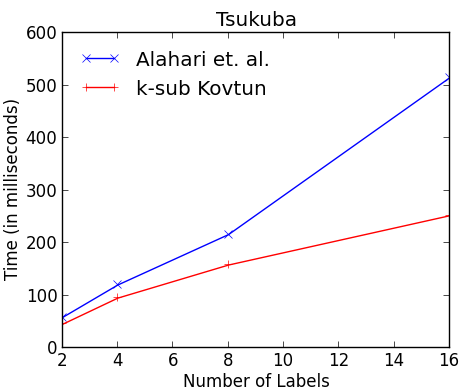} &
   \includegraphics[width=0.33\linewidth]{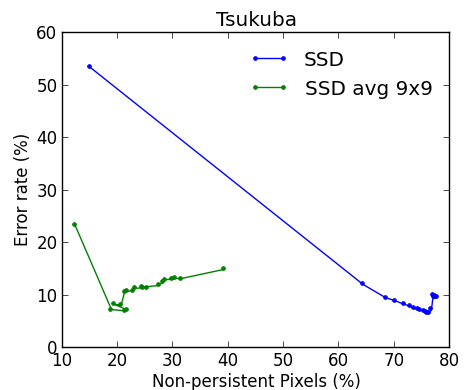} \\
\\
   \includegraphics[width=0.33\linewidth]{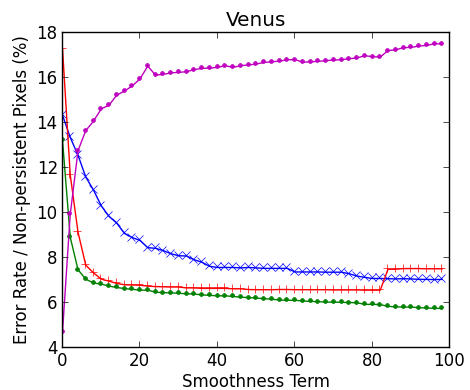} &
   \includegraphics[width=0.33\linewidth]{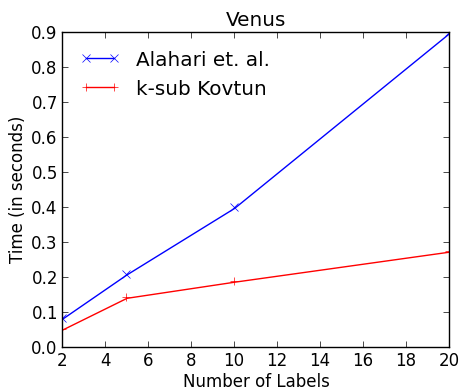} &
   \includegraphics[width=0.33\linewidth]{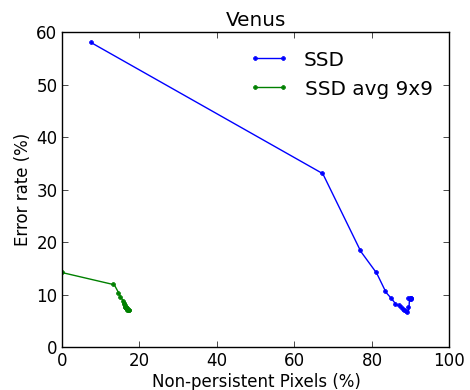} \\

   (a) & (b) & (c)
\end{tabular}

    \caption{
Results for images ``Teddy'', ``Cones'', ``Tsukuba'', ``Venus'': (a) dependency of the error rate on the smoothness term $\lambda$, (b) comparison of run-times of the 'Reduce' approach of \cite{Alahari:PAMI10} and $k$-sub Kovtun, (c) effect of the average costs aggregation - data points correspond to different values of the smoothness term from $0$ to $100$.}
   \label{fig:main1}
\end{figure*}

\begin{figure*}
\begin{tabular}{@{\hspace{-2pt}}c@{\hspace{4pt}}c@{\hspace{4pt}}c} 
   \includegraphics[width=0.33\linewidth]{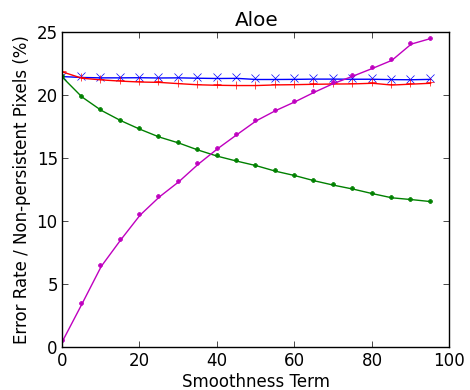} &
   \includegraphics[width=0.33\linewidth]{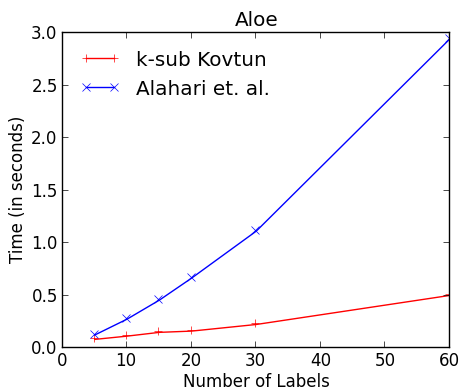} &
   \includegraphics[width=0.33\linewidth]{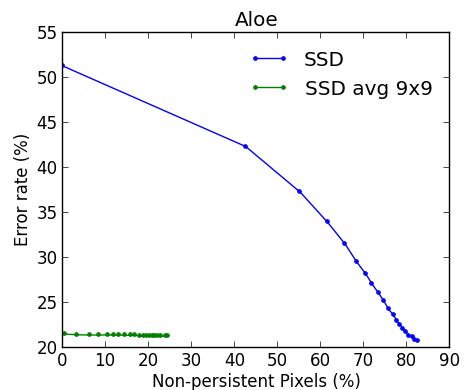} \\
\\
   \includegraphics[width=0.33\linewidth]{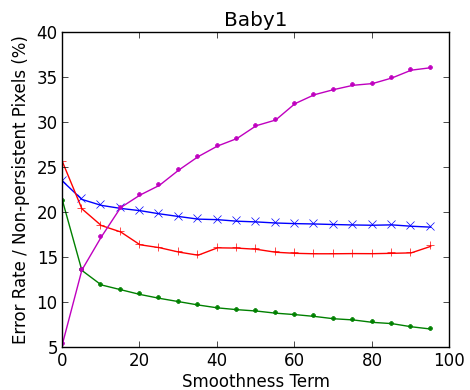} &
   \includegraphics[width=0.33\linewidth]{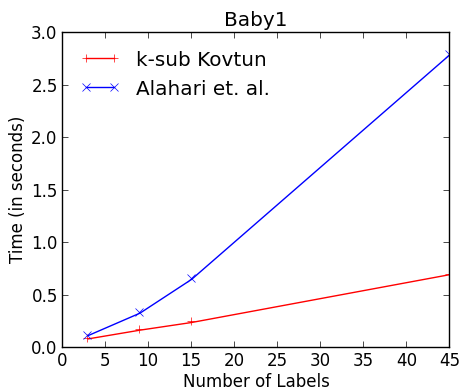} &
   \includegraphics[width=0.33\linewidth]{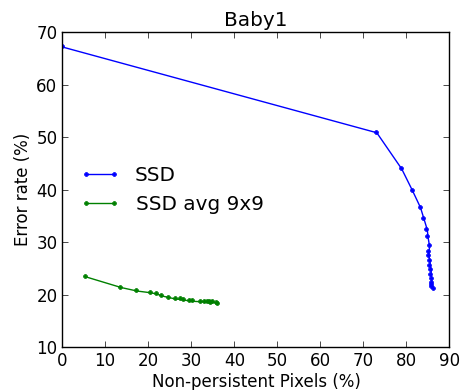} \\
\\
   \includegraphics[width=0.33\linewidth]{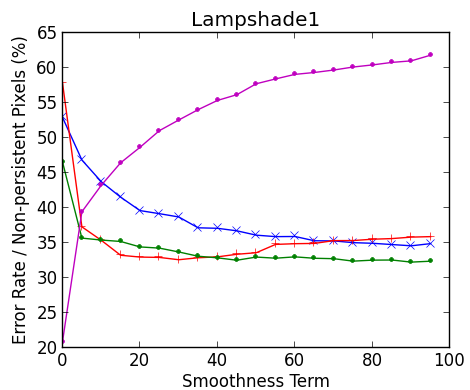} &
   \includegraphics[width=0.33\linewidth]{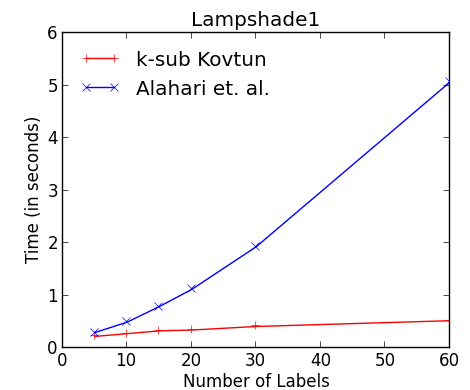} &
   \includegraphics[width=0.33\linewidth]{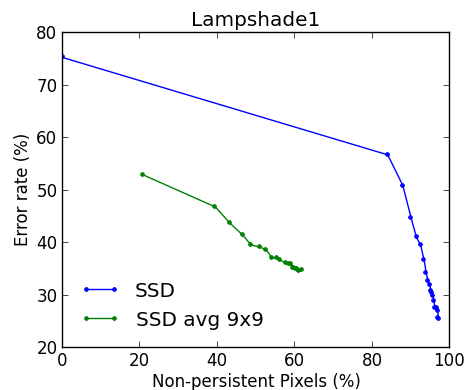} \\
\\
   \includegraphics[width=0.33\linewidth]{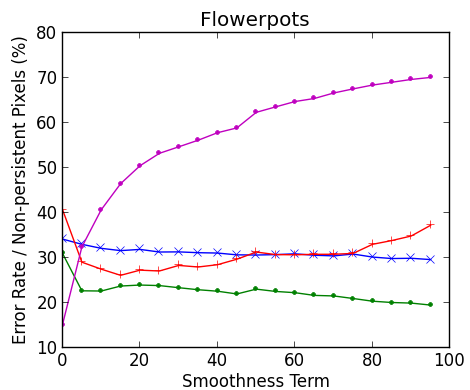} &
   \includegraphics[width=0.33\linewidth]{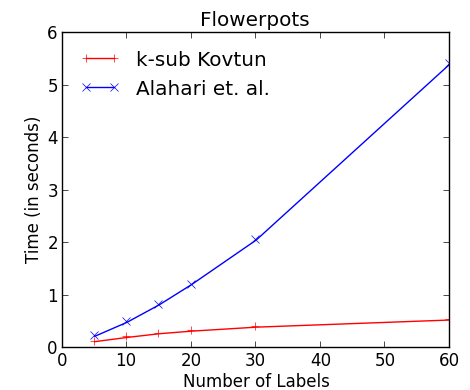} &
   \includegraphics[width=0.33\linewidth]{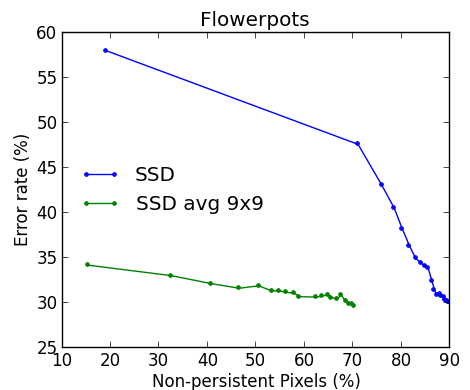} \\

      (a) & (b) & (c)
\end{tabular}
 
    \caption{Results for the images ``Aloe'', ``Baby1'', ``Lampshade1'' and ``Flowerpots'':
Results for images ``Teddy'', ``Cones'', ``Tsukuba'', ``Venus'': (a) dependency of the error rate on the smoothness term $\lambda$, (b) comparison of run-times of the 'Reduce' approach of \cite{Alahari:PAMI10} and $k$-sub Kovtun, (c) effect of the average costs aggregation - data points correspond to different values of the smoothness term from $0$ to $100$.}
   \label{fig:main2}
\end{figure*}

\else

\begin{figure*}[!t]
\begin{tabular}{@{\hspace{-2pt}}c@{\hspace{4pt}}c@{\hspace{4pt}}c}
   \includegraphics[width=0.33\textwidth]{img/paper/2_2_Teddy-Lambda-Non-persistent-100_d_2} &
   \includegraphics[width=0.33\textwidth]{img/paper/1_2_Teddy-Number_of_labels-Time} &
   \includegraphics[width=0.33\textwidth]{img/paper/3_2_Avg_eff_Teddy-Bad-Non-persistent-100_d_5} \\
(a) & (b) & (c)
\end{tabular}
 
    \caption{Results for images ``Teddy'': (a) dependency of the error rate on the smoothness term $\lambda$, (b) comparison of run-times of the 'Reduce' approach of \cite{Alahari:PAMI10} and $k$-sub Kovtun, (c) effect of the average costs aggregation - data points correspond to different values of the smoothness term from $0$ to $100$.
{\bf Results for other stereo pairs are given in the suppl. material.} The plots exhibit a similar behaviour, except that  occasionally the ER of the Kovtun's labeling
becomes worse than the ER of alpha-expansion for a sufficiently large $\lambda$ (plots (a)),
and for Lampshade1 and Aloe
the best ER of SSD was lower than the best ER of aggregated SSD (plots (c)).\vspace{-6pt}
}
   \label{fig:main}
\end{figure*}

\fi
%%%%%%%%%%%%%%%%%%%%%%%%%%%%%%%%%%%%%%%%%%%%%%%%%%%%%%%%%%%%%%%%%%%%%%%%

\myparagraph{Matching costs} The number of labeled pixels strongly depends on the method for computing matching costs $f_i(\cdot)$ and on the
regularization parameter $\lambda$ (which is the same for all edges). We tested the SSD matching costs and SSD cost averaged over the $9\times 9$ window
centered at pixel $i$. The latter method  gave a lower error\footnote{As in~\cite{Scharstein2002}, we define the error rate 
as the percentage of pixels whose predicted label differs from the ground truth label by more than 1.}
in 6 out of 8 images
\ifTR
(see Fig.~\ref{fig:main1},\ref{fig:main2}(c))
\else
(see Fig.~\ref{fig:main}(c))
\fi
 and labeled significantly more pixels. We thus used aggregated SSD costs for all experiments.

\myparagraph{Regularization parameter}  The effect of $\lambda$ is shown in 
\ifTR
Fig.~\ref{fig:main1},\ref{fig:main2}(a). 
\else
Fig.~\ref{fig:main}(a). 
\fi
Larger values of $\lambda$ typically give fewer labeled pixels.
%The error rate is usually worse if $\lambda$ is too small or too large.
 For subsequent experiments
we fixed $\lambda=20$ (which is also the default value in the stereo package that comes with~\cite{Scharstein2002}); this value appears to work well for most of the images.

\begin{table}
{\small
\begin{center}
\begin{tabular}{|l|c|c|c|c|}
\hline
Image(\# labels) & Alahari & $k$-sub     & $k$-sub  & \% non- \\
      & et al. &           & +FastPD & persistent \\
\hline\hline
Teddy(60)      & 3423 & 320   & 1016 & 18.1 \\
Cones(60)      & 3858 & 243   & ~466  & ~6.8  \\
Tsukuba(16)    & ~519 & 254   & ~469  & 19.3 \\
Venus(20)      & ~903 & 266   & ~570  & 16.0 \\
Lampshade1(60) & 5006 & 523   & 3850 & 48.8 \\
Aloe(60)       & 2786 & 236   & ~819  & 10.5 \\
Flowerpots(60) & 5492 & 568   & 3489 & 50.5 \\
Baby1(45)      & 2766 & 285   & 1095 & 22.0 \\
\hline
\end{tabular}
\end{center}
}
\vspace{-5pt}
\caption{Runtimes (in milliseconds) and \% of unlabeled pixels.\vspace{-10pt}}
\label{table:runtime}
\end{table}

\myparagraph{Speed comparisons} The speed of different algorithms is given in Table~\ref{table:runtime}. % (more results are given in the suppl. material).
$k$-sub Kovtun is approximately 10 times faster than the 'Reduce' method~\cite{Alahari:PAMI10}
(except for Venus and Tsukuba, which have fewer labels).
The fraction of non-persistent pixels ranged from 7\% to 50\%, which made the second phase
significantly faster.

We also tested how the running time of the first phase depends on the number of labels. For this experiment
we subsampled the set of allowed labels; the unary cost was set as the minimum over the interval that was merged to a given label.
Results are shown in
\ifTR
 Fig.~\ref{fig:main1},\ref{fig:main2}(b).
\else
 Fig.~\ref{fig:main}(b).
\fi
As expected, we get a larger speed-up with more labels.

It is reported in~\cite{Alahari:PAMI10} that the flow from Kovtun's computations speeds the alpha-expansion algorithm.
We were unable to replicate this in our implementation. However, we observed that initializing FastPD with the Kovtun's labeling
speeds it up compared to the ``$\ell_{\min}$-initialization'' \cite{Alahari:PAMI10}.\footnote{In this method we set $x_i\in\arg\min_a f_i(a)$. This
initialization was shown in~\cite{Alahari:PAMI10} to outperform the uniform initialization $x_i=0$.}
The average speed-up was 14.2\% 
\ifTR
(see Table~\ref{table:kiters}).
\begin{table}
\begin{center}
\begin{tabular}{|l|c|r|c|c|}
\hline
Image & $k$-sub Kovtun  & $l_{\min}$ & Speedup (\%)\\
\hline\hline
Teddy &696 &793&12.2\\
Cones &223 &261&14.6\\
Tsukuba &215 &269&20.1\\
Venus &304 &487&37.6\\
Lampshade1 & 3327 &4066&18.2\\
Aloe &299 &294&-1.7\\
Flowerpots & 2921&3257&10.3\\
Baby1&810&829&2.3\\
\hline
\end{tabular}
\end{center}
\caption{Running times of the second phase (in milliseconds) for different initializations of FastPD.
%and speedup of $k$-sub compared to $l_{min}$-init.
}
\label{table:kiters}
\end{table}
\else
(see suppl. material).
\fi

%(see Table X) 

\myparagraph{Quality of the Kovtun's labeling}
We found that
in the majority of cases Kovtun's labeling actually has a lower error rate
compared to the alpha-expansion solution (even though the energy of the latter is better) -
\ifTR
 see Fig.~\ref{fig:main1},\ref{fig:main2}(a).
 Disparity maps are shown in Fig.~\ref{fig:vis}.
\else
 see Fig.~\ref{fig:main}(a).
 Disparity maps are shown in the suppl. material.
\fi
Since computing Kovtun's labeling requires much less computation time, we argue
that it could be used in time-critical applications.

Not surprisingly, Kovtun's labeling is more reliable in the labeled part,
i.e.\ the error rate over persistent pixels is lower compared to the rate over the entire image 
\ifTR
 (Fig.~\ref{fig:main1},\ref{fig:main2}(a)).
\else
 (Fig.~\ref{fig:main}(a)).
\fi
Thus, for applications that require higher accuracy one might use an alternative technique for
the unlabeled part.

\ifTR
\begin{figure*}
\begin{tabular}{@{\hspace{4pt}}l@{\hspace{4pt}}c@{\hspace{4pt}}c@{\hspace{4pt}}c}
                Kovtun's labeling: &  &  &\vspace{-6pt} \\
                \includegraphics[width=0.24\linewidth]{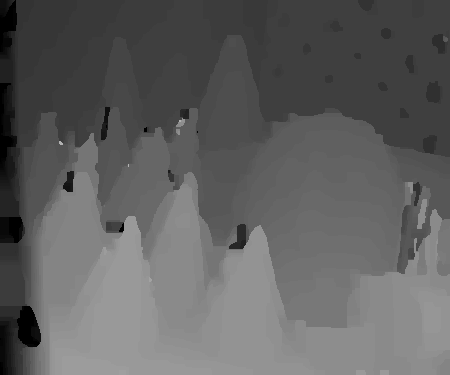} & 
                \includegraphics[width=0.24\linewidth]{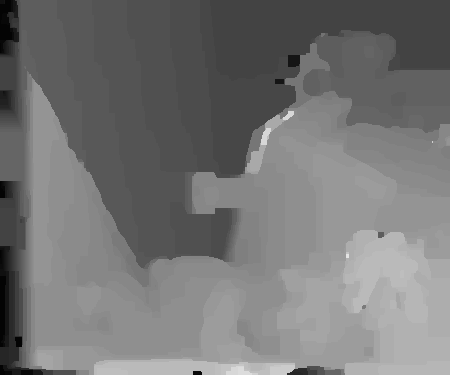} &
                \includegraphics[width=0.24\linewidth]{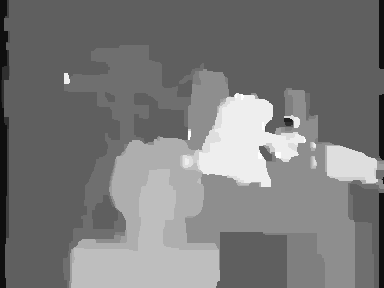} &
                \includegraphics[width=0.24\linewidth]{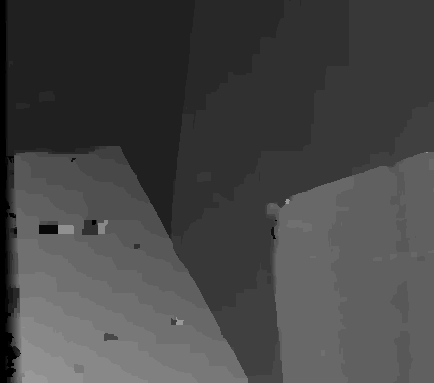} \\

                alpha expansion: &  &  &\vspace{-6pt} \\
                \includegraphics[width=0.24\linewidth]{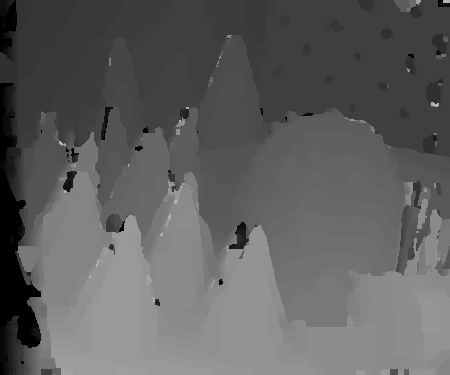} & 
                \includegraphics[width=0.24\linewidth]{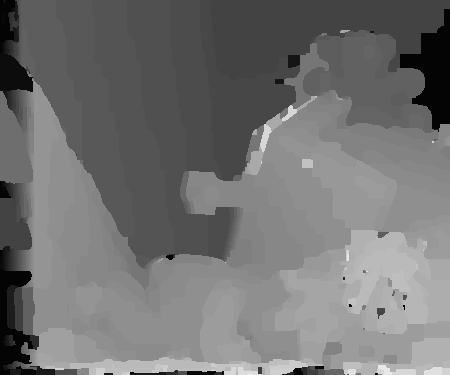} &
                \includegraphics[width=0.24\linewidth]{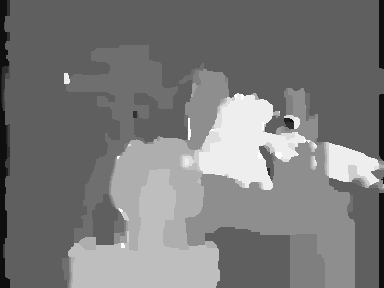} &
                \includegraphics[width=0.24\linewidth]{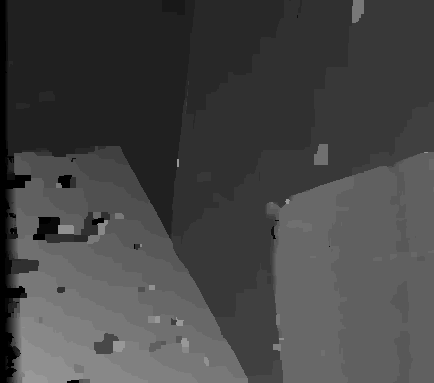} \\

                ground truth: &  &  & \vspace{-6pt} \\
                \includegraphics[width=0.24\linewidth]{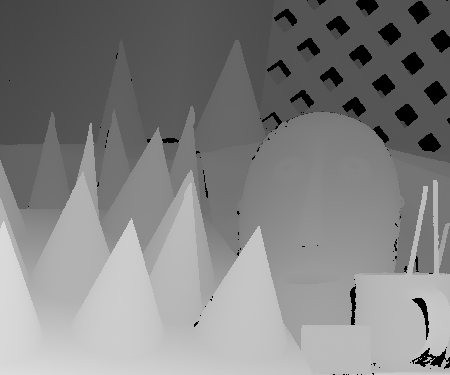} & 
                \includegraphics[width=0.24\linewidth]{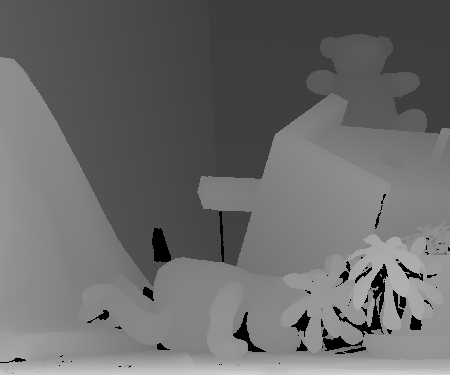} &
                \includegraphics[width=0.24\linewidth]{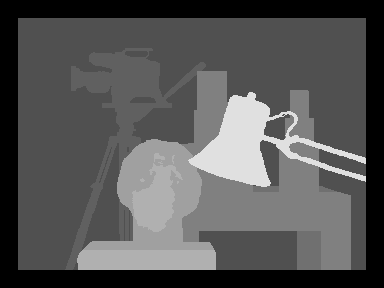} &
                \includegraphics[width=0.24\linewidth]{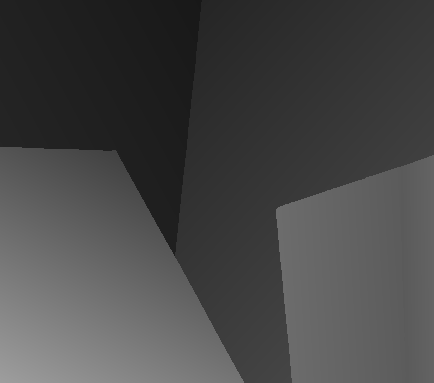} 
\end{tabular}

        \caption{Top row: output of $k$-sub Kovtun. Middle row: output of $k$-sub Kovtun+FastPD. Bottom row: ground truth.}
        \label{fig:vis}
\end{figure*}
\fi

%%%%%%%%%%%%%%%%%%%%%%%%%%%%%%%%%%%%%%%%%%%%%%%%%%%%%%%%%%%%%%%%%%%%%%%%%%%%%%%%%%%%%%%%%%%%%%

\ifTR
\myparagraph{Maximizing persistency} We observed that applying $k$-sub Kovtun multiple times decreases 
the number of non-persistent pixels (Fig.~\ref{fig:multk}). However, the decrease is relatively modest
(0.5\%-2\%). 

Another way to increase persistency is to augment the $k$-sub Kovtun routine.
Sometimes minimum $s$-$t$ cut problems may have multiple optimal solutions;
different choices may lead to different number of labeled pixels.
Since our method also gives optimal flows, in a post-processing step
we can traverse the graph for each label (in $O(|V|+|E|)$ time)
and compute the cut that gives the largest number of labeled nodes.
We call this procedure MP (``maximize persistency''). It has the biggest effect for $\lambda=0$; the numbers are given in Table \ref{table:mp}.
The effect for $\lambda=20$ is much more modest. 

Both techniques require some extra computation time and do not offer dramatic gains, so we decided not to use them by default.

\begin{figure*}
\begin{tabular}{@{\hspace{-2pt}}c@{\hspace{4pt}}c@{\hspace{4pt}}c} 
                \includegraphics[width=0.33\textwidth]{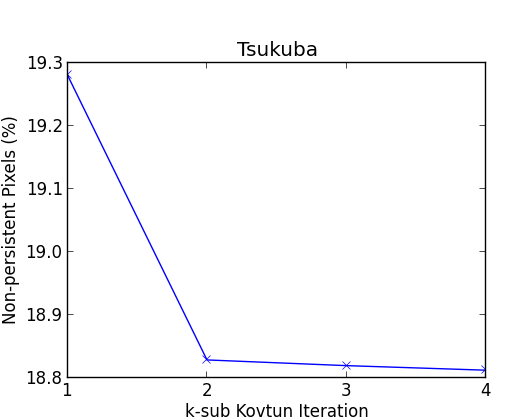} &
                \includegraphics[width=0.33\textwidth]{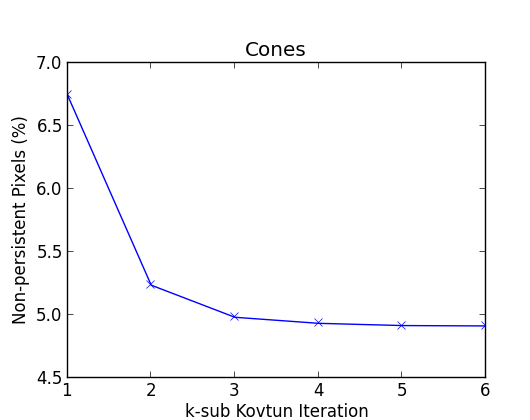} &
                \includegraphics[width=0.33\textwidth]{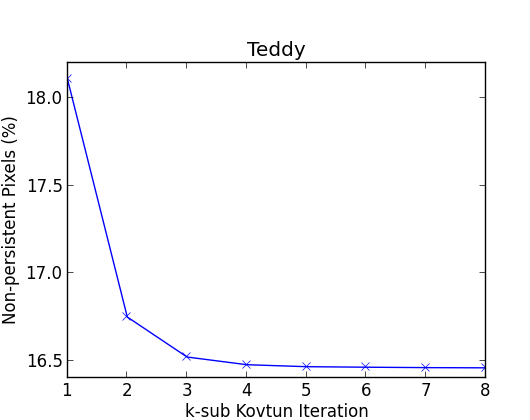} 
\end{tabular}
        \caption{Decrease in the number of non-persistent pixels with each iteration of $k$-sub Kovtun.}
        \label{fig:multk}
\end{figure*}

\begin{table}
\begin{center}
\begin{tabular}{|l|c|r|c|}
\hline
Image & $k$-sub& $k$-sub+MP & \% non-pers.\\
\hline\hline
Teddy & 179 & 230 & 2.1\\
Cones & 180 & 196 & 0.5\\
Tsukuba &  74 & 117 & 12.3\\ 
Venus & 124 & 166 & 4.7\\
Lampshade1 & 168 & 593 & 20.9 \\
Aloe & 175 & 190 & 0.5\\
Flowerpots & 180 & 619 & 15.2\\
Baby1 & 140 & 285 & 5.5\\
\hline
\end{tabular}
\end{center}
\caption{Effect of the MP procedure (``maximize persistency'') for $\lambda=0$.
The last column is the percentage of the non-persistent pixels given by $k$-sub Kovtun;
MP decreases this to 0.
%Running times (in milliseconds) of $k$-sub Kovtun with and without persistency maximisation procedure and percentage of the non-persistent nodes. MP stands for 'maximisation procedure'.
}
\label{table:mp}
\end{table}

\fi

\vspace{-7pt}
\section{Conclusions}\label{sec:conclusion}
\vspace{-3pt}
We see the contributions of this work as two-fold.
On the practical side, we showed how to improve the running time for the frequently
used Potts model. 
We tested it on the stereo problem (partly because there is an established dataset for that),
but we expect similar speed-ups for segmentation problems where labels correspond
to different semantic classes. If the number of persistent pixels is low for a given application then one
could use the cost aggregation trick to get more discriminative unary functions; as we saw for stereo,
this only improves the accuracy. For time-critical applications one could potentially skip the second phase
and use the Kovtun's labeling as the final output.

On the theoretical side, we introduced several concepts (such as $k$-submodular relaxations)
that may turn out to be useful for other energy functions. We hope that these concepts could lead to new directions
for obtaining partially optimal solutions for MAP-MRF inference.

\section*{Acknowledgements} We thank the authors of \cite{Alahari:PAMI10} for answering questions
about their implementation.

%\nocite{BVZ:PAMI01,Kohli:MQPBO,Kovtun:DAGM03,Kovtun:PhD,Shekhovtsov:CSC12}
%\nocite{ThapperZivny:FOCS12,Kolmogorov:BLP,ThapperZivny:dichotomy}
%\nocite{Kolmogorov:bisubmodularRelaxation,Felzenszwalb:TreeMetrics,Kolen,Kolmogorov:MFCS11,Huber:ISCO12}
%\nocite{Hochbaum:ACM01,Zalesky:JAM02,Chambolle:EMMCVPR05,Darbon:JMIV:I}

\renewcommand\refname{\large References\vspace{-2pt}}

{\small
\bibliographystyle{ieee}
\bibliography{vnk}
}

\end{document}